\documentclass[letterpaper, 10pt, conference]{ieeeconf}      

\IEEEoverridecommandlockouts                              

\usepackage{graphics} 
\usepackage{epsfig} 
\usepackage{amsmath} 
\usepackage{amssymb}  
\usepackage{subfigure}
\usepackage{verbatim} 
\usepackage{graphicx}

\usepackage{cite}
\usepackage{color}
\usepackage{graphicx}
\usepackage{bm}  

\usepackage{todonotes}
\usepackage[nolist,nohyperlinks]{acronym}

\usepackage{xcolor}

\def\FigControlSchemeFirstOrder{\centering\includegraphics[scale=0.2]{figures/FL-GP-model-based-first-order.png}}

\def\FigControlSchemeSecondOrder{\centering\includegraphics[scale=0.2]{figures/FL-GP-model-based-second-order.png}}

\def\FigTiltedPlane{\centering\includegraphics[scale=0.4]{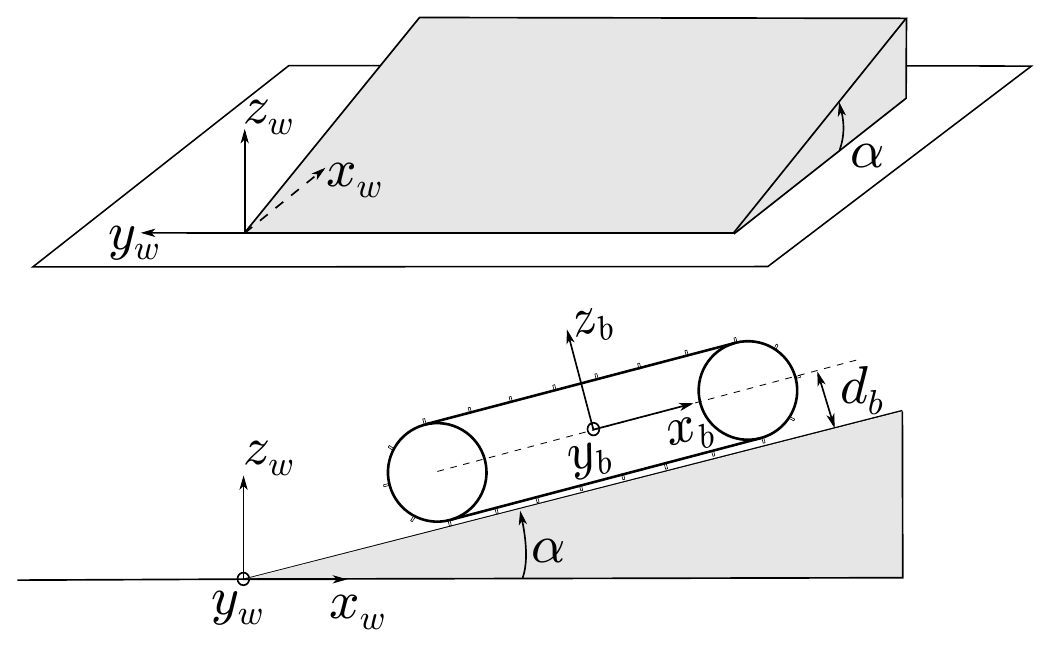}}

\newtheorem{theorem}{Theorem}[section]

\newtheorem{proposition}[theorem]{Proposition}
\newtheorem{corollary}[theorem]{Corollary}

\newcommand{\qed}{\nobreak \ifvmode \relax \else
      \ifdim\lastskip<1.5em \hskip-\lastskip
      \hskip1.5em plus0em minus0.5em \fi \nobreak
      \vrule height0.75em width0.5em depth0.25em\fi}

\def\qkp1{{\bf q}_{k{+}1}}

\def\qdkp1{\dot{{\bf q}}_{k{+}1}}

\def\xdkp1{\dot{x}_{k{+}1}}

\def\ydkp1{\dot{y}_{k{+}1}}

\def\yawdkp1{\dot{\theta}_{k{+}1}}

\def\vqdkp1{\begin{bmatrix}
\xdkp1 \\
\ydkp1 \\
\yawdkp1
\end{bmatrix}}

\def\ukp1{{\bf u}_{k{+}1}}

\title{\LARGE \bf
A Hybrid Approach for Trajectory Control Design}

\author{Luigi Freda$^{1}$ and Mario Gianni$^{1}$ and Fiora Pirri$^{1}$
\thanks{*This research is supported by EU-FP7-ICT-Project TRADR 609763}
\thanks{$^{1}$Authors are with ALCOR Laboratory, DIAG, Sapienza University of Rome, Italy
        {\tt\small \{freda, gianni, pirri\}@diag.uniroma1.it}}%
        }

\begin{document}

\maketitle
\thispagestyle{empty}
\pagestyle{empty}

\begin{abstract}
This work presents a methodology to design trajectory tracking feedback control laws, which embed non-parametric statistical models, such as Gaussian Processes (GPs). The aim is to minimize unmodeled dynamics such as undesired slippages. The proposed approach has the benefit of avoiding complex terramechanics analysis to directly estimate from data the robot dynamics on a wide class of trajectories. Experiments in both real and simulated environments prove that the proposed methodology is promising. 
%
%
\end{abstract}

\typeout{INTRODUCTION}
\section{Introduction}\label{sect:}
In the last decades, an increasing interest has been devoted to the design of high performance path tracking. 
In the literature, three main approaches to face this problem have emerged: \emph{(i)} model-based and adaptive control \cite{Martinez:2005,Ishigami:2006,Endo:2007,Moosavian:2008,Dar:2010}; \emph{(ii)} Gaussian Processes  or stochastic non-linear models for reinforcement learning of control policies \cite{Williams:2008,Deisenroth:2015}, and \emph{(iii)} nominal models and data-driven estimation of the residual \cite{Ko:2007,Ostafew:2014}.

The complexity of this problem requires several factors to be taken into account. First, theoretical understanding of the physical properties of the mechanical system, together with the physics underlying the interaction with the environment. Further, the identification of the system variables. Finally, uncertainty and noise in the measurements.

Motivated by the above considerations, in this work an hybrid approach is proposed, relying on both a derived nominal model and  Bayesian non-parametric data-driven estimation, based on Conditional Independent Gaussian Processes (CI-GPs). The nominal model serves to identify the leading variables of the system as well as to drive the process of acquisition of the data. 
The CI-GPs captures all the dynamics effects, which have not been explicitly accounted for by the nominal model. The CI-GPs manages uncertainty and noise in the measurements as well. 

In the hybrid method we propose here, the nominal model governs the leading physical variables, while the unforeseen physical variables are estimated by observations. Differently from the above mentioned approaches, the proposed model can capture dynamics effects, not designated  by the nominal model, by improving the feedback  with the required quantities provided by the stochastic model.  

For example, Endo {\em et al.} in \cite{Endo:2007} proposed an odometry model for tracked vehicles accounting for slippage. However, their method requires the accurate measurement of the vehicle velocity using internal sensors (e.g., an inertial sensor) and friction coefficients that change according to the ground surface. Moreover, model-based approaches are sensible to accumulated long-term errors.

Differently from model-based controllers, learned models resort to data collection to estimate the true robot model.
One of the main line of research, under this view,  is policy learning, either with stochastic non-linear models \cite{Ng:2000} or Gaussian Processes \cite{Deisenroth:2015}. 
However, the drawback of these approaches is that they require a huge amount of training examples for policy generation, though  they have the advantage of minimizing the error in the long run. 

The approaches based on the estimation of the residual assume that the real robot model is a linear combination of a prior and a disturbance model. The latter is then estimated  from experience in order to compensate for effects not captured by the prior, such as environmental disturbances and unknown dynamics \cite{Ostafew:2014}. These approaches, however, do not account  for the superimposition of the effects of the two models when linearly combined together, nor they take care of possible conflicts and misalignment between the individual estimates.  

The work is organized as follows. Next section introduces the nominal model. In section~\ref{sect:rob_contr} we introduce the trajectory control design and in Section~\ref{sect:cigp} we describe how the CI-GPs is grounded into the feedback control schema.
In Section~\ref{sect:experiments} we report the performance of the presented trajectory tracking control law incorporating the estimated model both in real and simulated environments.  Finally, we conclude with directions for future work.

%
%



\typeout{ROBOT MODEL}

\section{Robot Model}\label{Sect:RobotModel}

This section presents the first and second order kinematic models. 
We identify the main model variables and their general functional relationship. 

The robot moves in a 2D world. Its configuration is described by $\bm{q} = [x~y ~\varphi]^T \in SE(2)$, where $\bm{x} = [x~y]^T \in \mathbb{R}^2$ represents the robot position and $\varphi \in S^1$ is its yaw orientation.  
A complete 3D robot model is presented in sect.~\ref{Sect::3DKinematicModel}. 

\subsection{First Order Kinematic Model}\label{Sect:FirstOrderKinematics}

The 2D kinematic  model of a nonholonomic mobile robot is generally described by the following nonlinear driftless system~\cite{Siciliano:2010,Endo:2007}
\begin{equation}\label{Eq:KinematicsForwardDifferentialGeneral}
\dot{\bm{q}} = \bm{G}(\bm{q})\bm{v}
\end{equation}
where $\bm{v} = [v_l~v_r]^T \in \mathbb{R}^2$ is the vector of \emph{pseudo-velocities}\footnote{The difference between pseudo-velocities $\bm{v}$ and \emph{generalized velocities} $\dot{\bm{q}}$ will be clearer in Sect.~\ref{Sect:Sect:SecondOrderKinematics}, which presents the general form of the dynamic model equations for mobile robots.} containing the right and left track velocities\footnote{The most frequent definition of the pseudo-velocity vector is $\bm{v} = [v~\omega]^T \in \mathbb{R}^2$ where $v$ and $\omega$ are respectively the linear and angular velocities of the robot. Since $[v_l~v_r]^T = T [v~\omega]^T$, where $T \in GL(2)$ is an invertible 2x2 matrix depending on constant parameters, the above formulation is equivalent.}, and $\bm{G}(\bm{q})$ is a 3$\times$2 matrix whose columns span the null space of the matrix $\bm{A}^T(\bm{q})$ associated to the Pfaffian constraints\footnote{Kinematics constraints can be compactly expressed by using the Pfaffian form $\bm{A}^T(\bm{q})\dot{\bm{q}} = 0$ where the columns of $\bm{A}(\bm{q})$ are assumed to be smooth and linearly independent.}.
In practice, eq.~(\ref{Eq:KinematicsForwardDifferentialGeneral}) is typically specialized as
\begin{equation}\label{Eq:KinematicsForwardDifferential}
\dot{\bm{q}} = \bm{G}(\varphi)\bm{v}
\end{equation}
where it is implicitly assumed that the robot generalized velocities $\dot{\bm{q}}$ and the system equation parameters are independent from the robot position $\bm{x}$. This kinematics equation form is due to the the pure rolling constraint\footnote{For wheeled mobile robots the pure rolling constraint can be expressed as 
\[
 [{\sin\varphi}~{-\cos\varphi}~0]\dot{\bm{q}} {=} 0.
\] 
For tracked vehicles this constraint can be applied to each track in a slight different form by typically using the slip ratios: these characterize the longitudinal slips of the left and right tracks. See more in Sect.~\ref{Sect:TrackedVehicleKinematics}} which is typical of unicycle-like kinematics, and to the common assumption of an almost homogeneous supporting terrain. 

The forward integration of eq.~(\ref{Eq:KinematicsForwardDifferential}) by using proprioceptive sensory data is commonly referred to as \emph{dead reckoning}~\cite{Siciliano:2010}. This is typically used in the form of a first order difference equation obtained as
\begin{equation}\label{Eq:KinematicForwardModel}
\Delta\bm{q}_{t} = T_s \bm{G}(\varphi_t)\bm{v}_t = \bm{f}(\varphi_t,\bm{v}_t) = 
\begin{bmatrix}
f_x(\varphi_t,\bm{v}_t) \\
f_y(\varphi_t,\bm{v}_t) \\
f_\varphi(\varphi_t,\bm{v}_t)
\end{bmatrix} 
\end{equation}
where $t \in \mathbb{N}$ denotes the time index, $\Delta\bm{q}_{t} = \bm{q}_{t+1} - \bm{q}_t$ is the finite difference of $\bm{q}_{t}$ at time $t$, $T_s$ is a sufficiently small sample time that allows to approximate $\dot{\bm{q}}_{t} \simeq  \Delta\bm{q}_{t}/T_s$. We will refer to eq.~(\ref{Eq:KinematicForwardModel}) as the \emph{first order forward model}. 

The \textit{inverse model} of eq.~(\ref{Eq:KinematicForwardModel}) can be written as
\begin{equation}
\bm{v}_{t} = \bm{g}(\Delta\bm{q}^d_{t}, \varphi_t)
\end{equation}
which allows to compute the velocity commands $\bm{v}_{t}$ which are required to obtain a desired $\Delta\bm{q}^d_{t}$ for a given robot orientation $\varphi_t$. In principle, an inverse model can be obtained as 
\begin{equation}\label{Eq:DotVfromDotQ}
\bm{v} = \bm{G}^+(\varphi)\dot{\bm{q}}=(\bm{G}^T(\varphi)\bm{G}(\varphi))^{-1}\bm{G}(\varphi)^T\dot{\bm{q}}
\end{equation}
where we assume ${{\rm rank}(\bm{G}(\varphi))=2}$ and $\bm{G}^+$ is computed as a left pseudoinverse so as to guarantee $\bm{G}^+ \bm{G} = \bm{I}$. 
It follows that the composition $g \circ f$ coincides with the identity function ${\rm id}_{\bm{v}}$ on $\bm{v}$. For consistency we also require the inverse condition $f \circ  g = {\rm id}_{\Delta \bm{q}}$ to hold. In particular, we require 
\begin{equation}\label{Eq:FirstOrderConsistency}
\Delta\bm{x}^d_{t} = 
\begin{bmatrix}
f_x(\varphi_t,\bm{g}(\Delta\bm{q}^d_{t}, \varphi_t)) \\
f_y(\varphi_t,\bm{g}(\Delta\bm{q}^d_{t}, \varphi_t)) 
\end{bmatrix}.
\end{equation}
where we omit the condition on $\Delta\varphi^d_t$ since the $x$ and $y$ coordinates are flat outputs\footnote{I.e. the angle $\varphi$ can be computed as a function of the time derivatives of the component $x$ and $y$~\cite{Siciliano:2010}.} for the considered class of nonholonomic mobile robots. It is worth noting that, as detailed in sect.~\ref{Sect:FirstOrderControl}, eq.~(\ref{Eq:FirstOrderConsistency}) allows the implementation of a feedback linearization control scheme.
\begin{proposition}\label{Prop:FirstOrderConsistency}
Equation~(\ref{Eq:FirstOrderConsistency}) holds if and only if
\begin{equation}\label{Eq:FirstOrderConsistencyCondition}
\bm{G}(\varphi) \bm{G}^+(\varphi) = 
\begin{bmatrix}
\bm{I}_{2 \times 2} & \bm{0}\\
... & ... 
\end{bmatrix}.
\end{equation}
\end{proposition}
\begin{proof}
If eqs~(\ref{Eq:KinematicForwardModel}) and (\ref{Eq:DotVfromDotQ}) are plugged into eq.~(\ref{Eq:FirstOrderConsistency}), eq.~(\ref{Eq:FirstOrderConsistencyCondition}) immeditely follows.  
\end{proof}
Hereafter, we will use the subscript $K$ to refer to the first order forward model $\bm{f}_K(\cdot)$ and its inverse model $\bm{g}_K(\cdot)$. 

\subsubsection{Tracked Vehicles}\label{Sect:FirstOrderTrackedVehicle}

A simple but effective extension of the differential drive kinematic model adopts the following velocity transformations in order to include slippage effects
\begin{align}
v &=  (v_l + v_r)/2  \\
\omega &= \chi(v_r - v_l)/d .
\end{align}
where $\chi \in[0,1]$, aka the \textit{steering efficiency}~\cite{Martinez:2005}, acts as a damping factor on $\omega$ and $d$ is the vehicle tread. The resulting tracked vehicle kinematic model is
\begin{equation}\label{Eq:TrackedVehicleKinematics}
\dot{\bm{q}}
= \bm{G}_{\chi}(\varphi) \bm{v}=
\begin{bmatrix}
\frac{\cos\varphi}{2} & \frac{\cos\varphi}{2} \\ 
\frac{\sin\varphi}{2} & \frac{\sin\varphi}{2}\\ 
-\frac{\chi}{d} & \frac{\chi}{d}
\end{bmatrix}
\bm{v}
\end{equation}
and its inverse model is 
\begin{equation}\label{Eq:TrackedVehicleKinematicsInv}
\bm{v}
= 
\bm{G}^+_{\chi}(\varphi) \dot{\bm{q}}=
\begin{bmatrix}
\cos\varphi & \sin\varphi & -\frac{d}{2\chi} \\ 
\cos\varphi & \sin\varphi & +\frac{d}{2\chi}
\end{bmatrix}
\dot{\bm{q}}
\end{equation}
It is worth noting that eqs~(\ref{Eq:TrackedVehicleKinematics})--(\ref{Eq:TrackedVehicleKinematicsInv})  describe a classic unicycle model when $\chi = 1$. Moreover, $\bm{G}^+_{\chi}$ does not satisfy\footnote{Indeed, this is not a surprise since unicycle-like models cannot be transformed into a linear controllable systems by using a static state feedback~\cite{Oriolo:2002}.}
eq.~(\ref{Eq:FirstOrderConsistencyCondition}). In order to solve this latter problem, a different representative point for the robot position can be chosen, namely
the point ${\bm{x}_B=[x_B~y_B]=[x + b \cos\varphi~~y + b \sin\varphi]}$, with $b\neq 0$, located along the sagittal axis of the vehicle at distance $b$ from robot centre $[x~y]$. In this case, the forward model is
\begin{equation}\label{Eq:TrackedVehicleKinematicsB}
\dot{\bm{q}}_B
{=} \bm{G}_b(\varphi) \bm{v}{=}
\begin{bmatrix}
\frac{\cos\varphi}{2}+\frac{\chi b\sin\varphi}{d} & \frac{\cos\varphi}{2}-\frac{\chi b\sin\varphi}{d} \\ 
\frac{\sin\varphi}{2}-\frac{\chi b\cos\varphi}{d} & \frac{\sin\varphi}{2}+\frac{\chi b\cos\varphi}{d}\\ 
-\frac{\chi}{d} & \frac{\chi}{d}
\end{bmatrix}
\bm{v}
\end{equation}
and $\bm{x}_B$ is no more subject to nonholonomic constraints. We select as inverse model 
\begin{equation}\label{Eq:TrackedVehicleKinematicsBInv}
\bm{v} 
{=} \bm{G}^+_b(\varphi) \dot{\bm{q}}_B {=}
\begin{bmatrix}
\cos\varphi+\frac{d \sin\varphi}{2 \chi b} & \sin\varphi-\frac{d \cos\varphi}{2 \chi b} & 0\\ 
\cos\varphi-\frac{d \sin\varphi}{2 \chi b} & \sin\varphi+\frac{d \cos\varphi}{2 \chi b} & 0\\ 
\end{bmatrix}
\dot{\bm{q}}_B 
\end{equation}
where $\bm{q}_B=[x_B~y_B~\varphi]$. It is easy to show that $\bm{G}^+_b \bm{G} = \bm{I}$ and $\bm{G}^+_b$ satisfies  eq.~(\ref{Eq:FirstOrderConsistencyCondition}). By using proposition~\ref{Prop:FirstOrderConsistency}, we obtain that eq.~(\ref{Eq:FirstOrderConsistency}) is satisfied with $\bm{q}$ and $\bm{x}$ respectively replaced by $\bm{q}_B$ and $\bm{x}_B$. Note that, by using $\bm{x}_B$, the inverse model assumes the particular form
\begin{equation}\label{Eq:FirstOrderInverseModelWithOffset}
\bm{v}_{t} = \bm{g}_K(\Delta\bm{x}^d_{B,t}, \varphi_t)
\end{equation}
In the reminder of this paper, we will make use of this first order kinematic model.


\subsection{Second Order Kinematic Model}\label{Sect:Sect:SecondOrderKinematics}

The 2D dynamic model  of a mobile robot can be generally described by the following state-space reduced model~\cite{Siciliano:2010}
\begin{equation}\label{Eq:dynamicsGeneral}
\begin{cases}
\dot{\bm{q}} &= \bm{G}(\varphi) \bm{v}\\ 
\dot{\bm{v}} &= \bm{F}(\bm{q}, \bm{v},\bm{\tau})
\end{cases}
\end{equation}
where the first equation is the kinematic model presented in eq.~(\ref{Eq:KinematicsForwardDifferential}), $\bm{\tau} = [\tau_l~\tau_r]^T \in \mathbb{R}^2$ is the vector of left and right \emph{motor torques} and the function $\bm{F}(\cdot)$ represents the reduced dynamic model expressed in terms of pseudo-velocities. 

In general, commercial mobile robots come already equipped with two pre-tuned motor control systems. Each of these low-level control systems typically implements an independent control scheme around the controlled track/wheel. This control scheme is responsible of generating the motor torque signal $\tau_t \in \mathbb{R}$ in order to track the input \emph{reference velocity} $v^d_t \in \mathbb{R}$, and in order to guarantee asymptotic stability and disturbance rejection. Clearly, the dynamic response of the complete robot system depends on a large number of factors and on the particular operating conditions, and it is very difficult to capture it in an exact mathematical model. Given the structure of eq.~(\ref{Eq:dynamicsGeneral}) and assuming the two motor control system dynamics can be modelled as two independent and decoupled linear low-pass filters, we adopt the following simplified model
\begin{equation}\label{Eq:dynamics}
\begin{cases}
\dot{\bm{q}} &= \bm{G}(\varphi) \bm{v}\\ 
v_l^{(n)} &= \sum \limits_{i=0}^{n-1} a_i v_l^{(i)} + b_0 v_l^d\\
v_r^{(n)} &= \sum \limits_{i=0}^{n-1}  a_i v_r^{(i)} + b_0 v_r^d
\end{cases}
\end{equation}
where $v^{(n)}$ denotes the $n$-th time derivative of the \textit{actual velocity} $v$ (with $v^{(0)}\equiv v$) and the second and third equations model respectively the left and right low-pass-filter-like dynamics. Here, $n\in \mathbb{N}$ determines the overall \emph{model order}. Eq.~(\ref{Eq:dynamics}) can be considered as a $n+1$-order kinematic model. 

Note that the filter coefficients $a_i \in \mathbb{R}$ and $b_0 \in \mathbb{R}$ can be used in order to model a certain dynamic behaviour through a Butterworth, Tschebyscheff or Bessel filters. In general, a second order low pass filter (i.e. $n=2$) is typically sufficient for capturing the main dynamics characteristics: rise time, overshoot and settling time of the step response. It is worth noting that as in the previous section, we implicitly assume that the robot velocities and the parameters of eq.~(\ref{Eq:dynamics}) are independent from the robot position.

Now, assume the system dynamics bandwidth $B$ is sufficiently small w.r.t. sample frequency $1/T_s$\footnote{From Nyquist-Shannon sampling theorem, the following condition must hold $1/T_s > 2B$.}. By integrating the first part of eq.~(\ref{Eq:dynamics}) with Euler method and the second part with high-order numerical methods~\cite{Holoborodko:2008}, one can obtain the following equation
\begin{equation}
\begin{cases}
\Delta\bm{q}_t &= T_s \bm{G}(\varphi_t) \bm{v}_t\\ 
\bm{v}_{t+n} &= \sum \limits_{i=0}^{n-1}  c_i \bm{v}_{t+i} + d_0 \bm{v}_t^d
\end{cases}
\end{equation}
where the coefficients $c_i \in \mathbb{R}$ and $d_0 \in \mathbb{R}$ are the digital counterparts of the above terms $a_i$ and $b_0$. 
Finally, we can obtain an overall input-output model of the following form
\begin{equation}
\Delta\bm{q}_{t+n} = \bm{f}(\Delta\bm{q}_{t},\Delta\bm{q}_{t+1},...,\Delta\bm{q}_{t+n-1}, \varphi_t,\bm{v}^d_t)
\end{equation}
where $\Delta\bm{q}_{t+i} = \bm{q}_{t+i+1} - \bm{q}_{t+i}$ and we used the fact that $\bm{v}_{t+i} = \bm{G}^+(\varphi_{t+i}) \Delta\bm{q}_{t+i}/T_s$ and $\varphi_{t+n} = \varphi_t + \sum \limits_{i=0}^{n-1} \Delta\varphi_{t+i}$. 

For $n=1$ we can obtain the following second order forward model 
\begin{equation}\label{Eq:DynamicForwardModel}
\Delta\bm{q}_{t+1} = \bm{f}(\Delta\bm{q}_{t}, \varphi_t,\bm{v}^d_t) = \begin{bmatrix}
\bm{f}_x(\Delta\bm{q}_{t}, \varphi_t,\bm{v}^d_t) \\
\bm{f}_y(\Delta\bm{q}_{t}, \varphi_t,\bm{v}^d_t) \\
\bm{f}_\varphi(\Delta\bm{q}_{t}, \varphi_t,\bm{v}^d_t)
\end{bmatrix}. 
\end{equation}
Here an exponentially weighted moving average can be used as first order low pass filter
\begin{equation}\label{Eq:ExpMovingAverageFilter}
\bm{v}_{t+1} = \alpha \bm{v}_t + (1-\alpha) \bm{v}_t^d
\end{equation}
where $\alpha \in [0,1]$ acts as a smoothing/forgetting factor. By using this filter, one obtains
\begin{equation}\label{Eq:DynamicForwardModelDetail}
\Delta{\bm{q}_{t+1}}
= T_s \bm{G}(\varphi_{t+1}) (\alpha \bm{G}^+(\varphi_t)\dfrac{\Delta\bm{q}_t}{T_s} + (1-\alpha) \bm{v}^d_t) 
\end{equation}
where $\varphi_{t+1} = \varphi_t + \Delta\varphi_t$.
It is worth noting that in the previous equations $\Delta\bm{q}_{t+1}$ is used to convey second order time derivative information since for small $T_s$: ${\ddot{\bm{q}}_{t} \simeq (\Delta \bm{q}_{t+1} - \Delta\bm{q}_{t})/T_s^2}$.

The inverse model of eq.~(\ref{Eq:DynamicForwardModel}) is 
\begin{equation}
\bm{v}^d_{t} = \bm{g}(\Delta\bm{q}^d_{t+1}, \Delta\bm{q}_{t}, \varphi_t)
\end{equation}
which allows to compute the reference velocities $\bm{v}^d_{t}$ which are required to obtain a desired $\Delta\bm{q}^d_{t+1}$ for given robot "velocity" $\Delta\bm{q}_{t}$ and orientation $\varphi_t$. Starting from eq.~(\ref{Eq:DynamicForwardModelDetail}) and using the left-pseudo inverse property $\bm{G}^+ \bm{G} = \bm{I}$, one obtains 
\begin{equation}\label{Eq:DynamInverseModelDetail}
\bm{v}^d_t
= \frac{1}{1-\alpha}( \bm{G}^+(\varphi_{t+1})\dfrac{\Delta{\bm{q}^d_{t+1}}}{T_s} - \alpha \bm{G}^+(\varphi_t)\dfrac{\Delta{\bm{q}_t}}{T_s}) 
\end{equation}
where again $\varphi_{t+1} = \varphi_t + \Delta\varphi_t$.
As in sect.~\ref{Sect:FirstOrderKinematics}, for consistency we require the condition $f \circ  g = {\rm id}_{\Delta \bm{q}}$ to hold. In particular, we require 
\begin{equation}\label{Eq:SecondOrderConsistency}
\Delta\bm{x}^d_{t+1} = \begin{bmatrix}
f_x(\Delta\bm{q}_{t}, \varphi_t,\bm{g}(\Delta\bm{q}^d_{t+1},\Delta\bm{q}_{t}, \varphi_t)) \\
f_y(\Delta\bm{q}_{t}, \varphi_t,\bm{g}(\Delta\bm{q}^d_{t+1},\Delta\bm{q}_{t}, \varphi_t))
\end{bmatrix} 
\end{equation}
where, as in Sect.~\ref{Sect:FirstOrderKinematics}, we omitted the condition on $\Delta\varphi^d_{t+1}$ given the flatness of the considered nonholonomic systems. 
In sect.~\ref{Sect:SecondOrderControl}, eq.~(\ref{Eq:SecondOrderConsistency}) will allow the implementation of feedback linearization scheme on the second order model.
\begin{proposition}\label{Prop:SecondOrderConsistency}
Equation~(\ref{Eq:SecondOrderConsistency}) holds if and  only if eq.~(\ref{Eq:FirstOrderConsistencyCondition}) is satisfied.
\end{proposition}
\begin{proof}
By plugging eqs~(\ref{Eq:DynamicForwardModelDetail}) and~(\ref{Eq:DynamInverseModelDetail}) into eq.~(\ref{Eq:SecondOrderConsistency}), one can easily obtain eq.~(\ref{Eq:FirstOrderConsistencyCondition}). 
\end{proof}

Hereafter, we will use the subscript $D$ to refer to the second order forward kinematic model $\bm{f}_D(\cdot)$ and its inverse model $\bm{g}_D(\cdot)$.

\subsubsection{Second Order Model for Tracked Vehicles}

In this work, we consider the second order model which is obtained by plugging the matrix $\bm{G}_b$ of eqs~(\ref{Eq:TrackedVehicleKinematicsB})--(\ref{Eq:TrackedVehicleKinematicsBInv}) into eqs~(\ref{Eq:DynamicForwardModelDetail}) and~(\ref{Eq:DynamInverseModelDetail}).
As shown above, eq.~(\ref{Eq:FirstOrderConsistencyCondition}) is satisfied by $\bm{G}_b$ and hence, by using proposition~\ref{Prop:SecondOrderConsistency},  eq.~(\ref{Eq:SecondOrderConsistency}) holds with $\bm{q}$ and $\bm{x}$ respectively replace by $\bm{q}_B$ and $\bm{x}_B$. Note that, by using $\bm{G}_b$,  we obtain the following form
\begin{equation}\label{Eq:SecondOrderInverseModelWithOffset}
\bm{v}^d_{t} = \bm{g}_D(\Delta\bm{x}^d_{B,t+1},\Delta\bm{q}_{B,t}, \varphi_t)
\end{equation}
In the reminder of this paper, we will make use of this second order model.

\begin{table}\label{Tab:ForwardInverseSummary}
\centering
\resizebox{\columnwidth}{!}{%
\begin{tabular}{|l|l|l|}
\hline
~ &  Forward & Inverse \\
\hline 
1st order & $\Delta\bm{q}_{B,t} = \bm{f}_K(\varphi_t,\bm{v}_t)$ & $\bm{v}_{t} = \bm{g}_K(\Delta\bm{x}^d_{B,t}, \varphi_t)$\\
\hline
2nd order & $\Delta\bm{q}_{B,t+1} = \bm{f}_D(\Delta\bm{q}_{B,t}, \varphi_t,\bm{v}^d_t)$ & $\bm{v}^d_{t} = \bm{g}_D(\Delta\bm{x}^d_{B,t+1},\Delta\bm{q}_{B,t}, \varphi_t)$ \\
\hline
\end{tabular}
}%
\caption{Adopted forward and inverse functional models.} 
\end{table}

\typeout{ROBOT CONTROL}
\begin{figure}[!t]
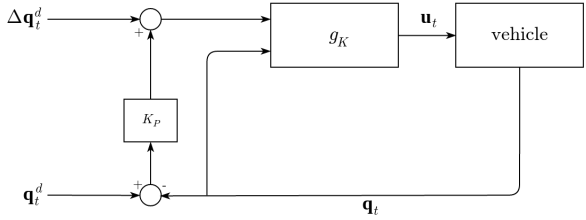

\begin{center}
\FigControlSchemeFirstOrder \caption{First order control scheme.} \label{Fig:FigFirstOrderControlScheme}
\end{center}
\end{figure}

\begin{figure}[!t]
\begin{center}
\FigControlSchemeSecondOrder \caption{Second order control scheme.} \label{Fig:FigSecondOrderControlScheme}
\end{center}
\end{figure}

\begin{figure}
\centering
\subfigure[]{\includegraphics[width=0.48\columnwidth]{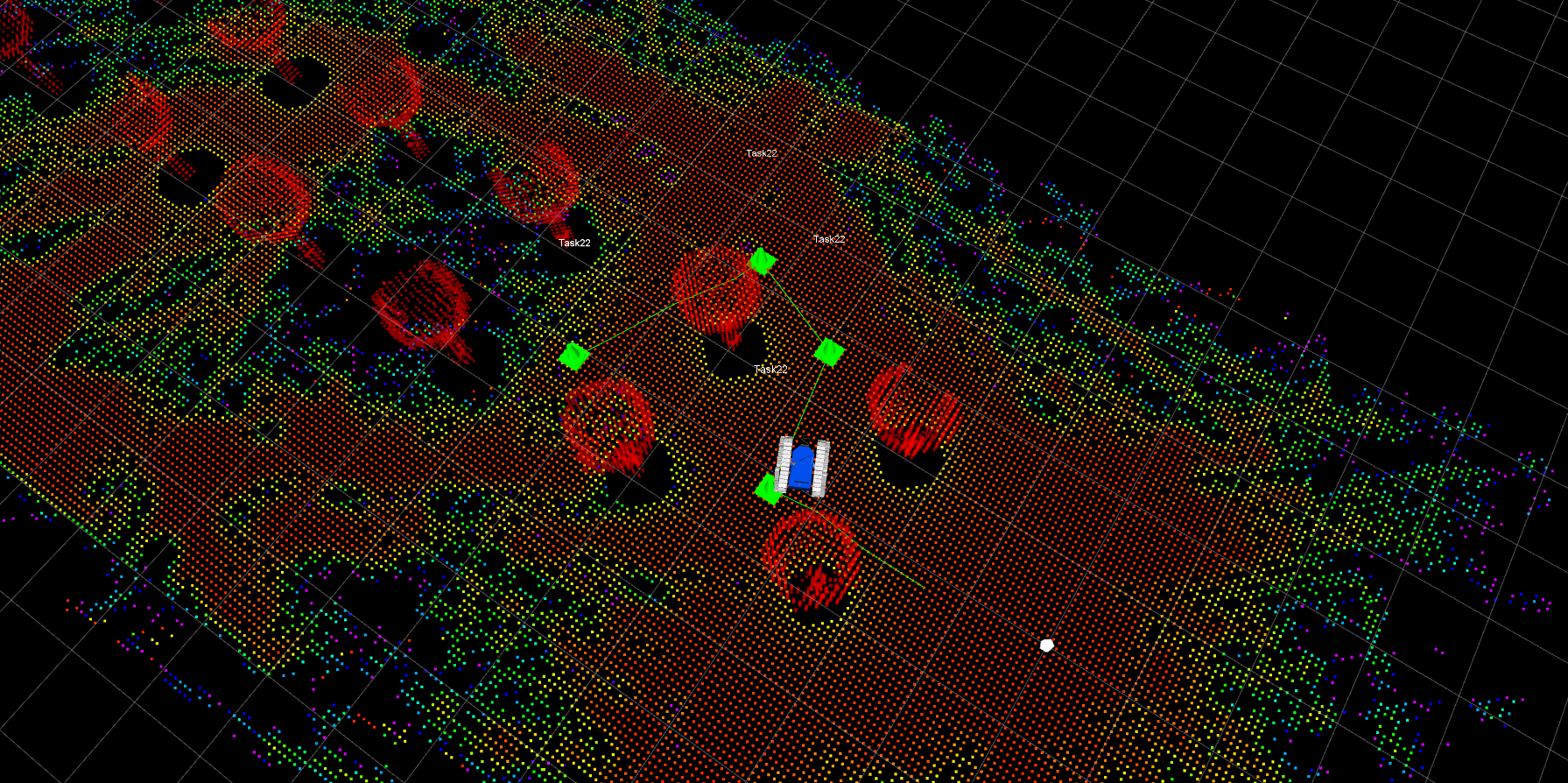}}
\subfigure[]{\includegraphics[width=0.48\columnwidth]{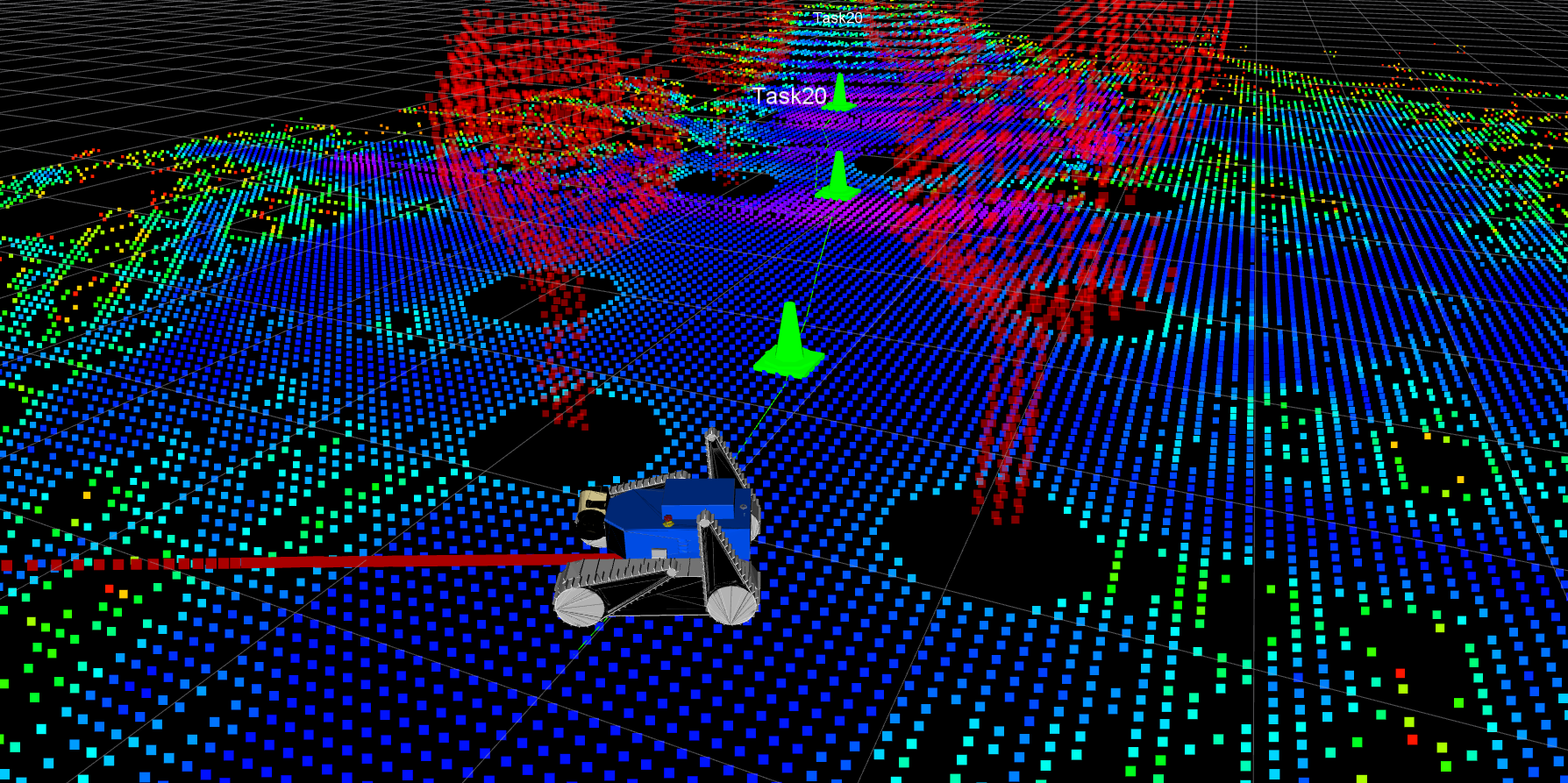}}
\subfigure[]{\includegraphics[width=0.48\columnwidth]{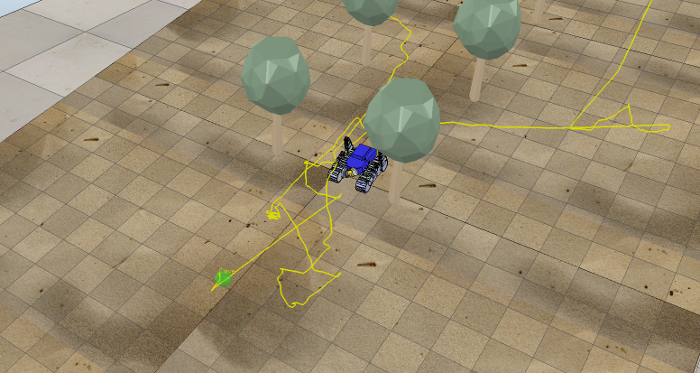}}
\subfigure[]{\includegraphics[width=0.48\columnwidth]{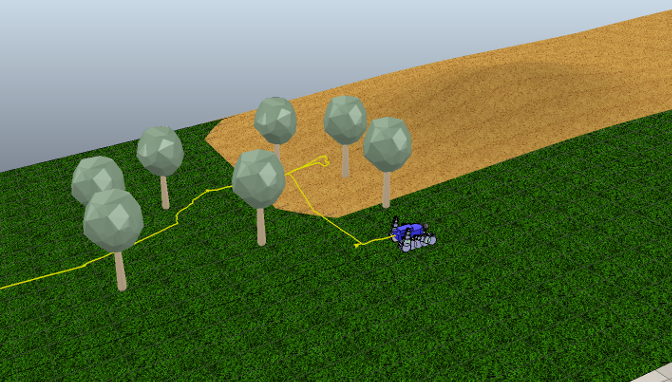}}
\caption{(a) and (b) Free-shape reference trajectories generated in simulation through a waypoints selection interface on a horizontal surface and on an inclined one with different friction coefficients, respectively; (c) Robot tracking the free-shape reference trajectory on a flat terrain; (d) Tracking task performed on an inclined surface with friction.}
\label{fig:data_coll}
\end{figure}

\section{Trajectory controller design}\label{sect:rob_contr}

\subsection{First Order Model Control}\label{Sect:FirstOrderControl}

Consider eqs~(\ref{Eq:TrackedVehicleKinematicsB})--(\ref{Eq:FirstOrderInverseModelWithOffset}) of sect.~\ref{Sect:FirstOrderKinematics}.  Since equation~(\ref{Eq:FirstOrderConsistency}) is satisfied by $\bm{G}^+_b$ the following corollary holds. 
\begin{corollary}
Consider eq.~(\ref{Eq:FirstOrderInverseModelWithOffset}), the control law
\begin{equation}\label{Eq:FirstOrderControlScheme}
\bm{v}_{t} = g_K(\bm{u}_t, \varphi_t) 
\end{equation} 
implements a feedback linearization control scheme~\cite{Isidori:1995} on the trajectory $\bm{x}_{B,t}$, i.e. one obtains $\Delta \bm{x}_{B,t} = \bm{u}_t$.
\end{corollary}
At simple trajectory tracking controller can be obtained with
\begin{align}\label{Eq:FirstOrderControlLaw}
\bm{u} &=\Delta\bm{x}^d_{t} + K_P(\bm{x}^d_{t} - \bm{x}_{B,t})
\end{align}
where $K_P = {\rm diag}(k_{P,i})$ is a 2x2 diagonal matrix which contains the scalar control gains $k_{P,i} \in \mathbb{R}$ for $i\in\{1,2\}$. A block diagram representing this control law is sketched in Fig.~\ref{Fig:FigFirstOrderControlScheme}. 
In fact, by plugging eqs~(\ref{Eq:FirstOrderControlScheme})--(\ref{Eq:FirstOrderControlLaw}) into eq.~(\ref{Eq:FirstOrderConsistency}) one obtains 
\begin{equation}
\Delta\bm{x}_{B,t} = \Delta\bm{x}^d_{t} + K_P(\bm{x}^d_{t} - \bm{x}_{B,t})
\end{equation}
which is characterized by the following discrete time error dynamics 
\begin{equation}
\textbf{e}_{t+1} = (1-K_P) \textbf{e}_{t} 
\end{equation}
where $\textbf{e}_{t} = \bm{x}^d_{t} - \bm{x}_{B,t}$. The error dynamics is asymptotically stable when the condition $\vert 1 - k_{P,i} \vert < 1$ is satisfied  for $i\in\{1,2\}$.

\subsection{Second Order Model Control}\label{Sect:SecondOrderControl}

Given the nominal model in eq. (\ref{Eq:SecondOrderInverseModelWithOffset}), the following control law
\begin{align}\label{Eq:SecondOrderControlScheme}
\bm{v}_{t} &= {\bf g}_D(\bm{u}_{t+1}, \Delta\bm{q}_{B,t},\varphi_t).
\end{align}
implements a feedback linearization control scheme on the trajectory $\bm{x}_{B,t}$. 
In fact a simple trajectory tracking controller can be obtained with
\begin{align}\label{Eq:SecondOrderControlLaw}
\bm{u}_{t+1} &=\Delta\bm{x}^d_{t+1} + K_D(\Delta\bm{x}^d_{t} - \Delta\bm{x}_{B,t}) + K_P(\bm{x}^d_{t} - \bm{x}_{B,t})
\end{align}
where $K_D = {\rm diag}(k_{D,i})$ and $K_P = {\rm diag}(k_{P,i})$, with $k_{D,i} \in \mathbb{R}$ and $k_{P,i} \in \mathbb{R}$, for $i\in\{1,2\}$, are the 2$\times$2 diagonal gain matrices of the controller.
A block diagram representing this control law is sketched in Figure~\ref{Fig:FigSecondOrderControlScheme}. 

If one plugs eqs~(\ref{Eq:SecondOrderControlScheme})--(\ref{Eq:SecondOrderControlLaw}) into eq.~(\ref{Eq:SecondOrderConsistency}), the following equation is obtained 
\begin{equation}
\Delta\bm{x}_{B,t+1} =\Delta\bm{x}^d_{t+1} + K_D(\Delta\bm{x}^d_{t} - \Delta\bm{x}_{B,t}) + K_P(\bm{x}^d_{t} - \bm{x}_{B,t}).
\end{equation}
By exploiting the fact that $\Delta\bm{x}^d_{t} -\Delta\bm{x}_{B,t} = \textbf{e}_{t+1} - \textbf{e}_{t}$, where $\textbf{e}_{t} = \bm{x}^d_{t} - \bm{x}_{B,t}$, one can immediately obtain the following error dynamics 
\begin{equation}
\textbf{e}_{t+2} + (K_D-1)\textbf{e}_{t+1} + (K_P-K_D) \textbf{e}_{t} = 0.
\end{equation} 
whose system poles are 
\begin{equation}
\lambda^i_{1,2} = \frac{ -(k_{D,i}-1) \pm \sqrt{(k_{D,i}-1)^2 - 4(k_{P,i} - k_{D,i})} }{2}  \hspace{0.5cm} 
\end{equation}
for $i \in \{1,2\}$. The previous equation can be exploited to suitably select the control gains $k_{D,i}$ and $k_{P,i}$ in order to guarantee the system asymptotic stability condition ${\vert \lambda^i_{1,2} \vert < 1}$ for $i\in\{1,2\}$.

\typeout{LEARNING}
\begin{figure}
\centering
\subfigure[]{\includegraphics[width=\columnwidth]{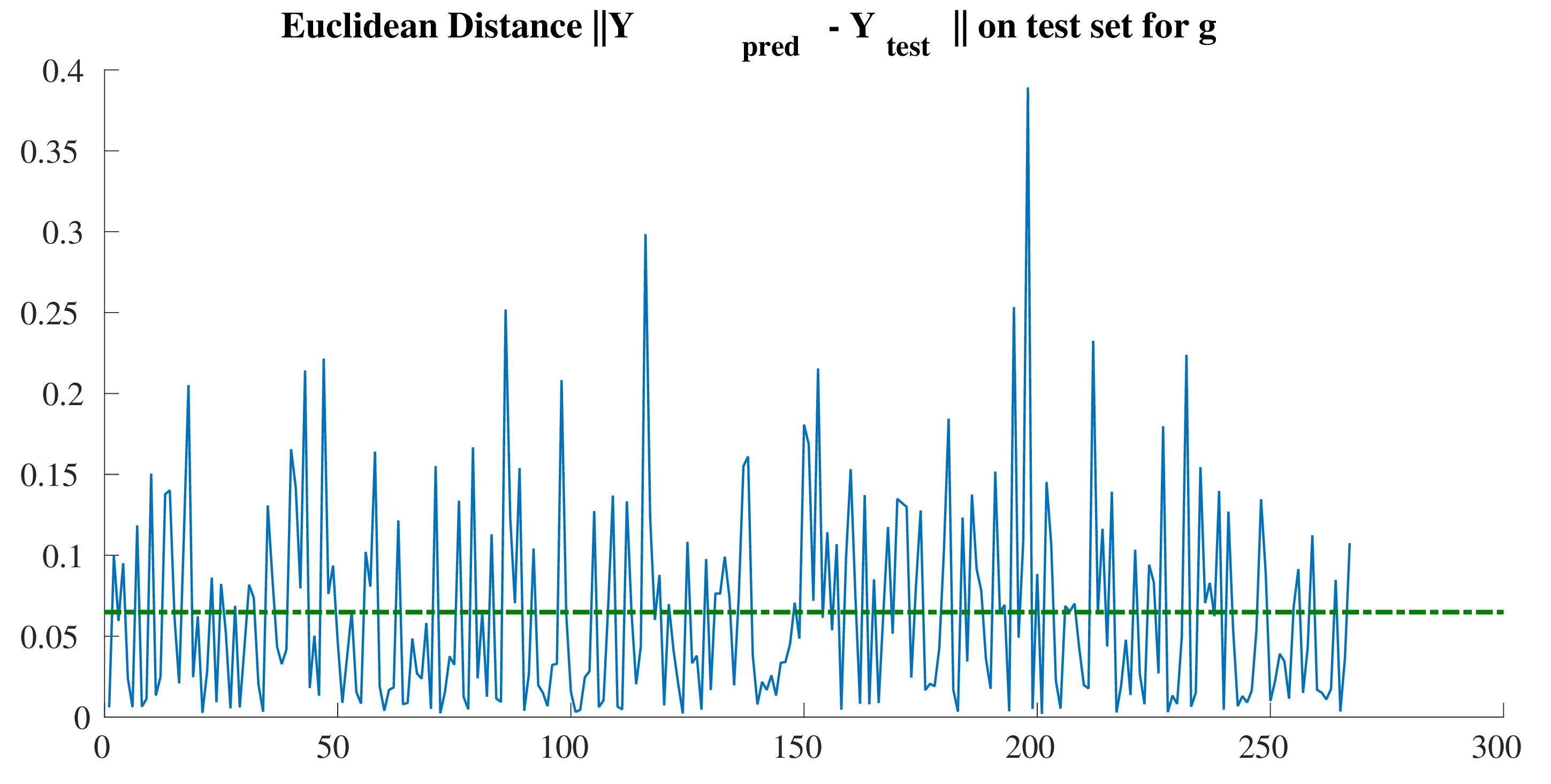}}
\subfigure[]{\includegraphics[width=\columnwidth]{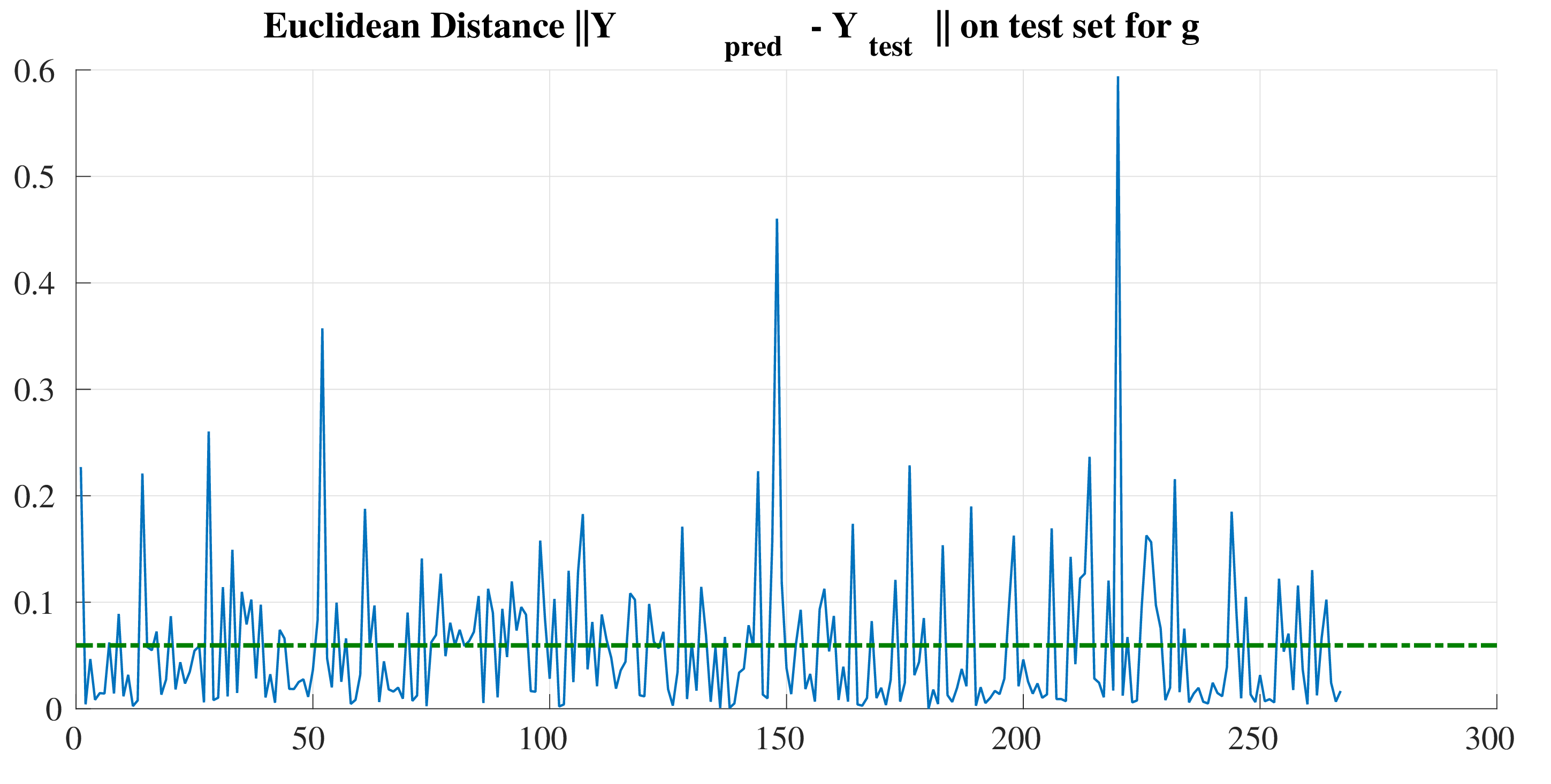}}
\caption{(a) Prediction accuracy on the test set of the CI-GPs with data collected while the robot is tracking a figure-8 reference trajectory on a horizontal surface. Average error $0.007$; (b) Prediction accuracy on the test set of the CI-GPs with data collected while the robot is tracking a free-form trajectory lying on a tilted plane with friction. Average error $0.059$.}
\label{fig:acc_pred}
\end{figure}

\section{CI-GP Model estimation}\label{sect:cigp}

In this section we consider the identification of the system variables given in eq.~(\ref{Eq:SecondOrderInverseModelWithOffset}) and  the nominal model control law given in eq.~(\ref{Eq:SecondOrderControlScheme}). Featuring the variables identified in the above equations and the observations from real and simulated data we provide a stochastic control model, which compensates the dynamic effects not explicitly accounted for in the nominal model. 

Let $\mathcal{D}{=}\{\langle{\bf z}_{i},\!{\bf w}_{i}\rangle\}_{i{=}1}^N$ be the set of observations, collected according to the identification of the system variables specified in eq. (\ref{Eq:SecondOrderInverseModelWithOffset}). Here
\begin{align}
{\bf z}_{i}{=}\begin{bmatrix}
v_l \\ v_r
\end{bmatrix}{\in}\mathbb{R}^2 \quad \text{and} \quad {\bf w}_{i}{=}\begin{bmatrix}
\Delta\bm{x}_{B,t+1} \\ \Delta\bm{q}_{B,t} \\ \varphi_t
\end{bmatrix}{\in}\mathbb{R}^6
\end{align}



Let $j\in \{1,2\}$, we assume that $z_{i,j}\in {\bf z}_i$ is the outcome of a latent function $\zeta_j\!:\mathbb{R}^{6}{\rightarrow}\mathbb{R}$ on ${\bf w}_{i}$, corrupted by Gaussian noise $\varepsilon_j$, that is
\begin{align}\label{eq:latentF}
z_{i,j} & =  \zeta_j\!\left({\bf w}_{i}\right) + \varepsilon_j \,\, \text{with} \,\, \varepsilon_j{\sim}\mathcal{N}(0,\sigma^2_j)
\end{align} 
for every $i{=}1,\ldots,N$.

A GP prior with zero mean and covariance function $\kappa_{j}\!\left(\cdot,\cdot\right)$ is placed over each latent function $\zeta_j\!\left(\cdot\right)$. More precisely
\begin{align}\label{eq:gp_prior}
\zeta_j\!\left({\bf w}\right) & \sim \text{GP}_j\!\left(0,\kappa_{j}\!\left({\bf w},{\bf w}'\right)\right)
\end{align}

On the basis of the assumptions in both eq. (\ref{eq:latentF}) and eq. (\ref{eq:gp_prior}) it follows that, for each $j{\in}\{1,2\}$
\begin{align}\label{eq:cigp}
{\bm \zeta}_j|{\bf W} \sim \mathcal{N}\!\left({\bf 0},{\bf K}_{j}\right)  \,\, \text{and} \,\, {\bf z}_{j}|{\bf W} & \sim \mathcal{N}\!\left({\bf 0},{\bf K}_j + \sigma^2_j{\bf I}\right)
\end{align}
with
\begin{align}
{\bm \zeta}_j{\triangleq} \begin{bmatrix} \zeta_j\!\left({\bf w}_{1}\right) \\ \vdots \\ \zeta_j\!\left({\bf w}_{N}\right) \end{bmatrix}\!, \quad {\bf z}_j{\triangleq}\begin{bmatrix} z_{1,j} \\ \vdots \\  z_{N,j} \end{bmatrix} \,\, \text{and} \,\, {\bf W}{=}\begin{bmatrix} {\bf w}_{1}^{\top} \\ \vdots \\ {\bf w}_{N}^{\top}\end{bmatrix}^{\top} \nonumber
\end{align}
${\bf K}_j$ is the kernel matrix whose entries ${\bf K}_j\!(u,v){=}\kappa_{j}\!\left({\bf w}_{u},{\bf w}_{v}\right)$.

Given a query input ${\bf w}_{\star}$, the model in eq.~(\ref{eq:cigp}) returns an estimate ${\bf z}_{\star}{\triangleq}\begin{bmatrix} z_{*,1} & \cdots & z_{*,M} \end{bmatrix}^{\top}$ such that
\begin{align}
z_{\star,j} = {\bm \kappa}_{\star,j}^{\top}\!\left({\bf K}_j{+}\sigma^2_j{\bf I}\right)^{{-}1}\!{\bf z}_j~, \quad\,j{\in}\{1,2\}
\end{align}
Here ${\bm \kappa}_{\star,j}{\triangleq}\begin{bmatrix}
\kappa_j\!\left({\bf w}_{1},\!{\bf w}_{\star}\right) & \cdots & \kappa\!\left({\bf w}_{N},\!{\bf w}_{\star}\right)\end{bmatrix}^{\top}$ is the vector of covariances between the query point ${\bf w}_{\star}$ and the $N$ training points in $\mathcal{D}$.

Finally, learning of the hyperparameters $\Theta_j$ of each GP$_j$ (which vary according to the chosen kernel function $\kappa_j\!\left(\cdot,\cdot\right)$), is performed separately and independently, for each component $j$, via the maximization of the marginal log-likelihoods of the outputs ${\bf z}_{j}$ given the inputs ${\bf w}_{1},\ldots,{\bf w}_{N}$, that is
\begin{align}\label{eq:cigp_opt}
\Theta_j^{\text{max}}{=}\underset{\Theta_j}{\mathrm{argmax}}\{(\log\!\left(p({\bf z}_{j}|{\bf W})\right)\} \,\,\text{for}\,\,j{\in}\{1,2\}
\end{align} 
with
\begin{align}\label{eq:loglikelihood}
\log\!\left(p({\bf z}_{j}|{\bf W})\right){=} & -\frac{1}{2}{\bf z}_{j}^{\top}\!\left({\bf K}_j{+}\sigma^2_j{\bf I}\right)^{{-}1}\!{\bf z}_j \nonumber \\
& \,\, {-}\frac{1}{2}\log\vert{\bf K}_j{+}\sigma^2_j{\bf I}\vert {-}\frac{N}{2}\log 2\pi
\end{align}
Maximization of the log likelihood in eq. (\ref{eq:loglikelihood}) can be done by using efficient gradient-based optimization algorithms such as conjugate gradients~\cite{Rasmussen:2005}

In Section~\ref{sect:experiments}, we report both the prediction accuracy of the CI-GP together with its performance within the control schema, which is described in the next section.


\typeout{RESULTS}

\begin{figure}
\centering
\subfigure[]{\includegraphics[height=4.5cm]{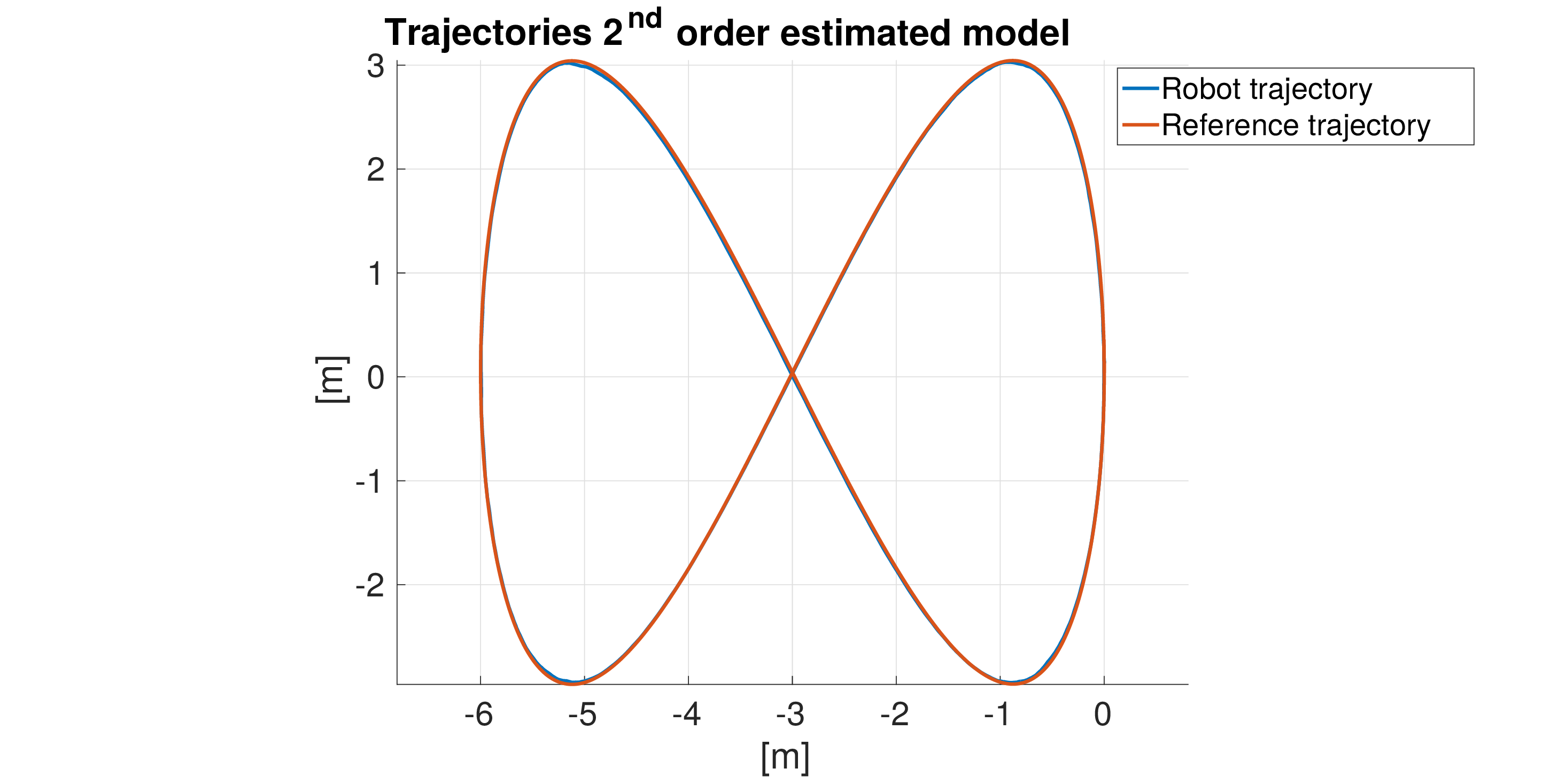}}\quad
\subfigure[]{\includegraphics[height=4.5cm]{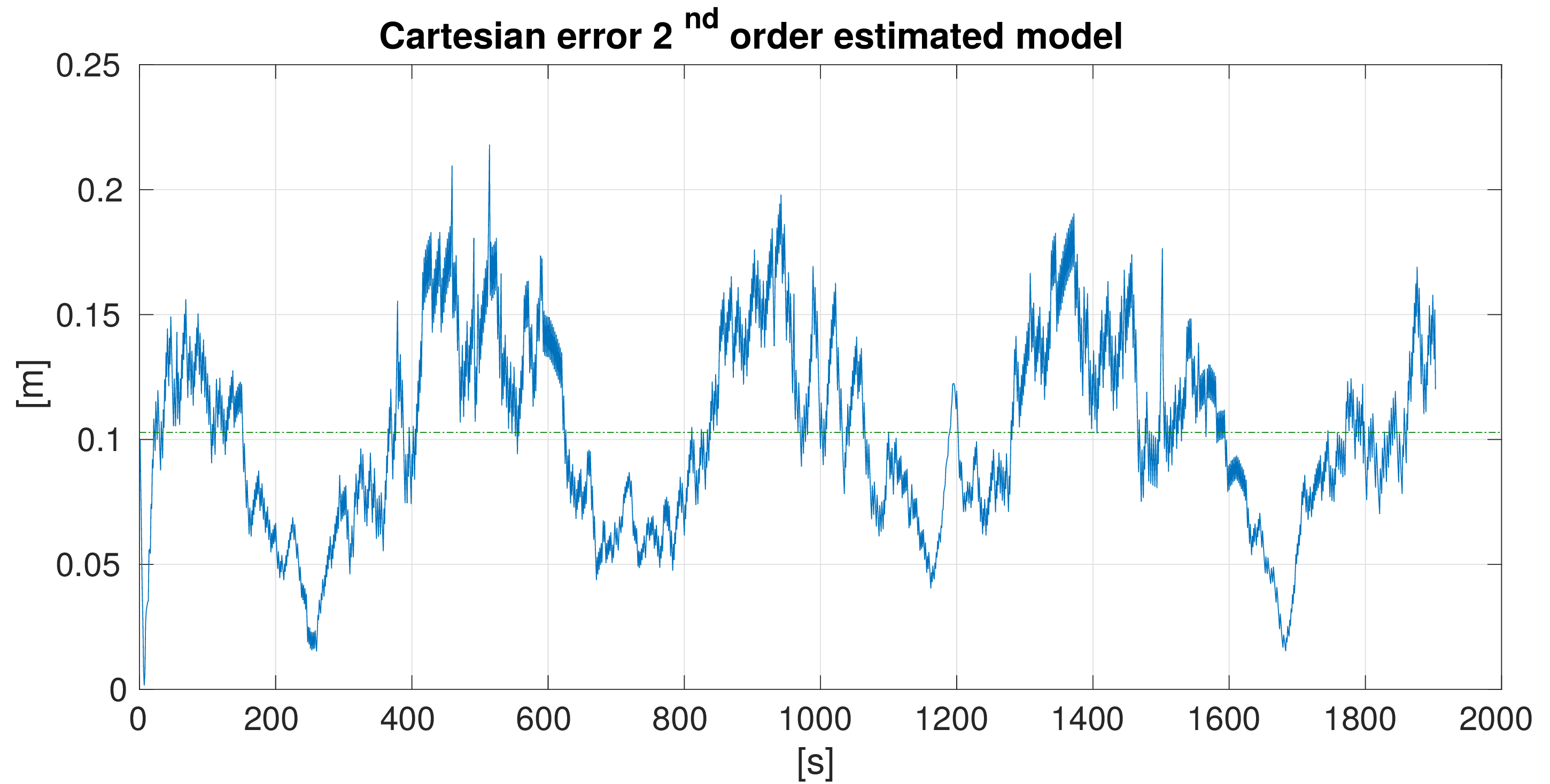}}
\caption{(a) Path effectively followed by the robot (blue line), driven by a feedback control low based on the second order model, estimated via GPs, with respect to the figure-8 reference trajectory (red line). (b) Cartesian error of the estimated model along the reference. Error average $0.1028$ $m$.}
\label{fig:case8}
\end{figure}

\begin{figure}
\centering
\subfigure[]{\includegraphics[height=4.5cm]{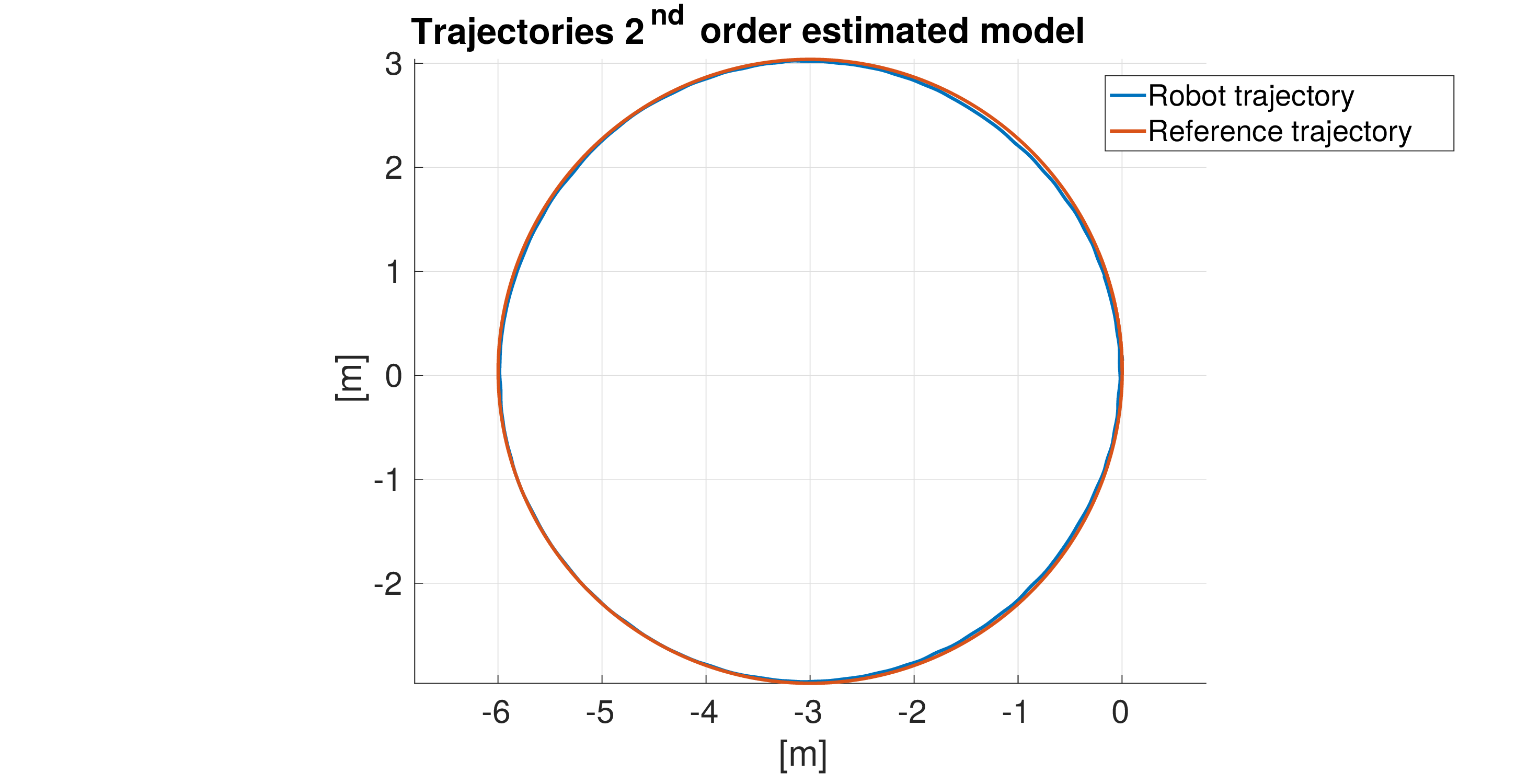}}\quad
\subfigure[]{\includegraphics[height=4.5cm]{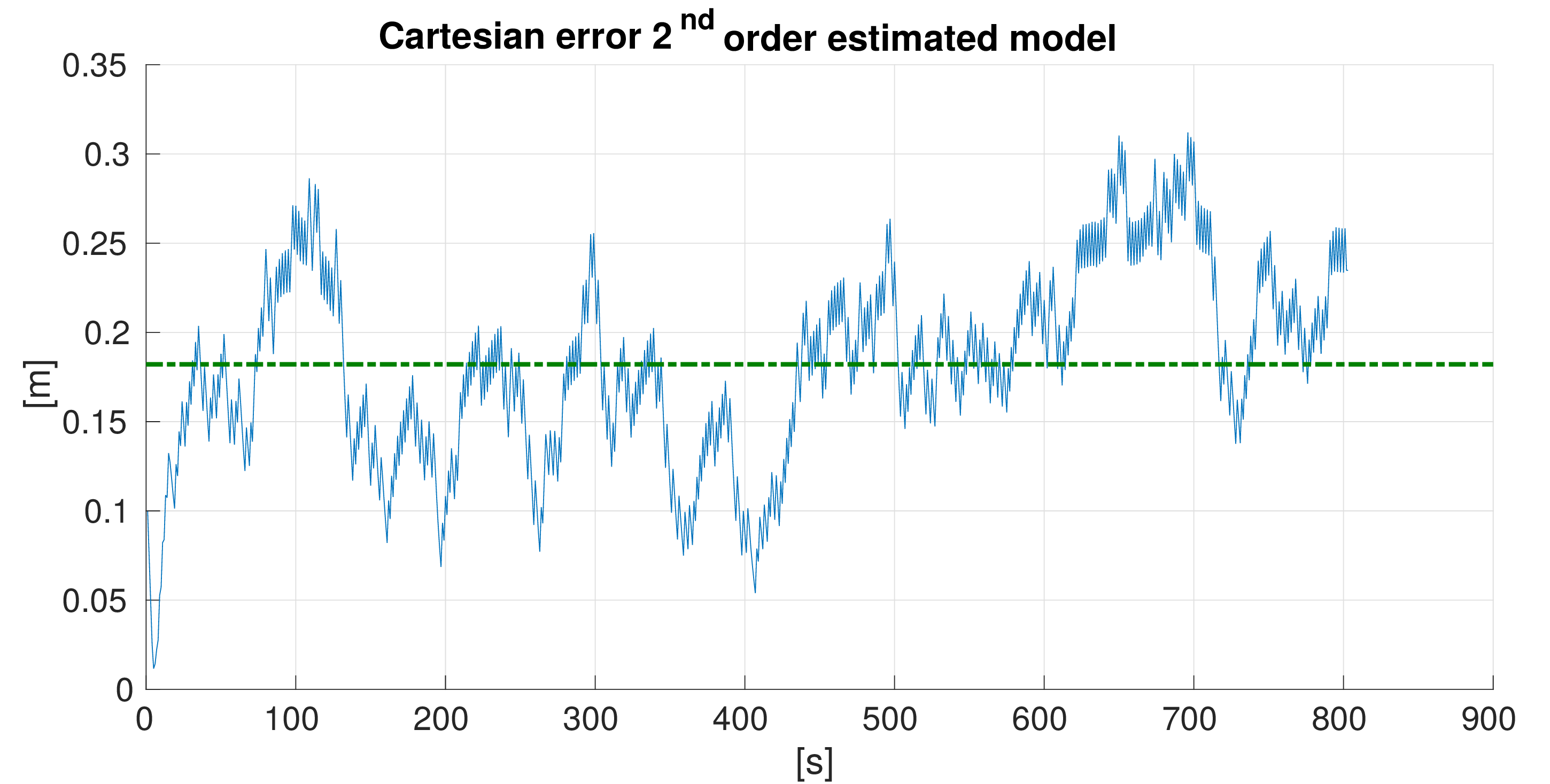}}
\caption{(a) Path effectively followed by the robot (blue line), driven by a feedback control low based on the second order model, estimated via GPs, with respect to the circular reference trajectory (red line). (b) Cartesian error of the estimated model along the reference. Error average $0.1821$ $m$.}
\label{fig:case10}
\end{figure}

\begin{figure}
\centering
\subfigure[]{\includegraphics[height=4.5cm]{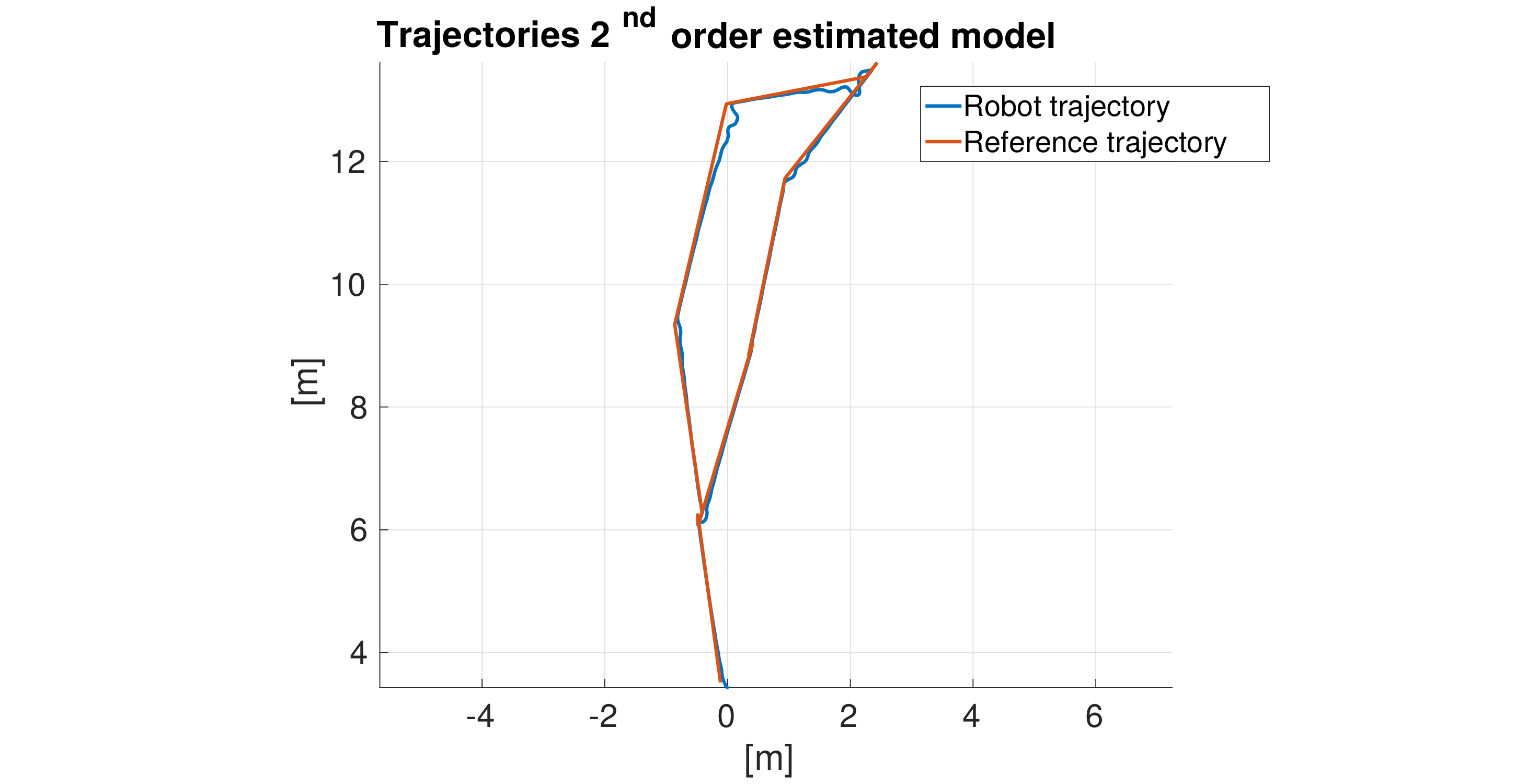}}\quad
\subfigure[]{\includegraphics[height=4.5cm]{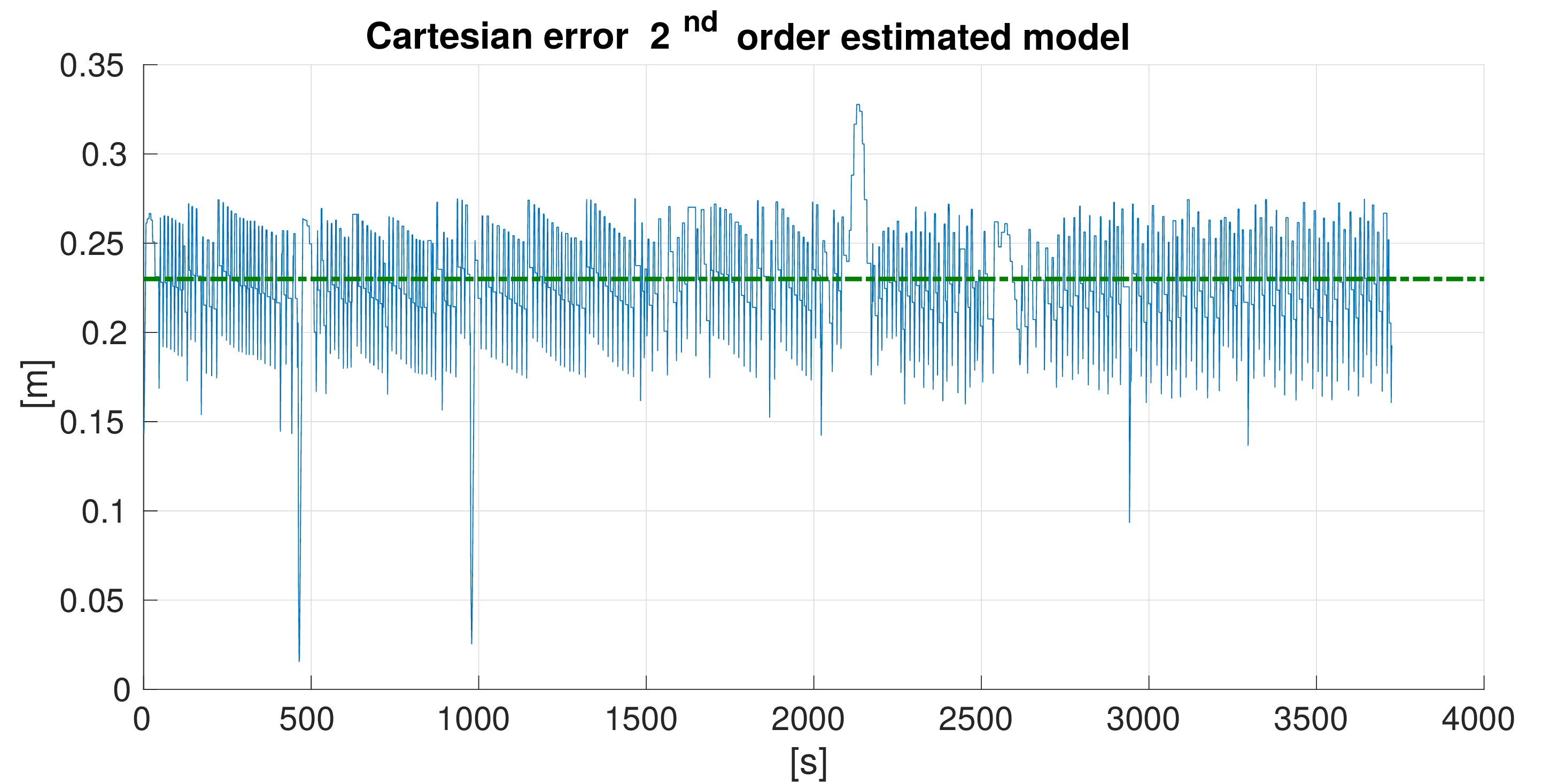}}
\caption{(a) Path effectively followed by the robot (blue line), driven by a feedback control low based on the second order model, estimated via GPs, with respect to a free-form reference trajectory (red line). (b) Cartesian error of the estimated model along the reference. Error average $0.23$ $m$.}
\label{fig:case12}
\end{figure}

\begin{figure}
\centering
\subfigure[]{\includegraphics[height=4.5cm]{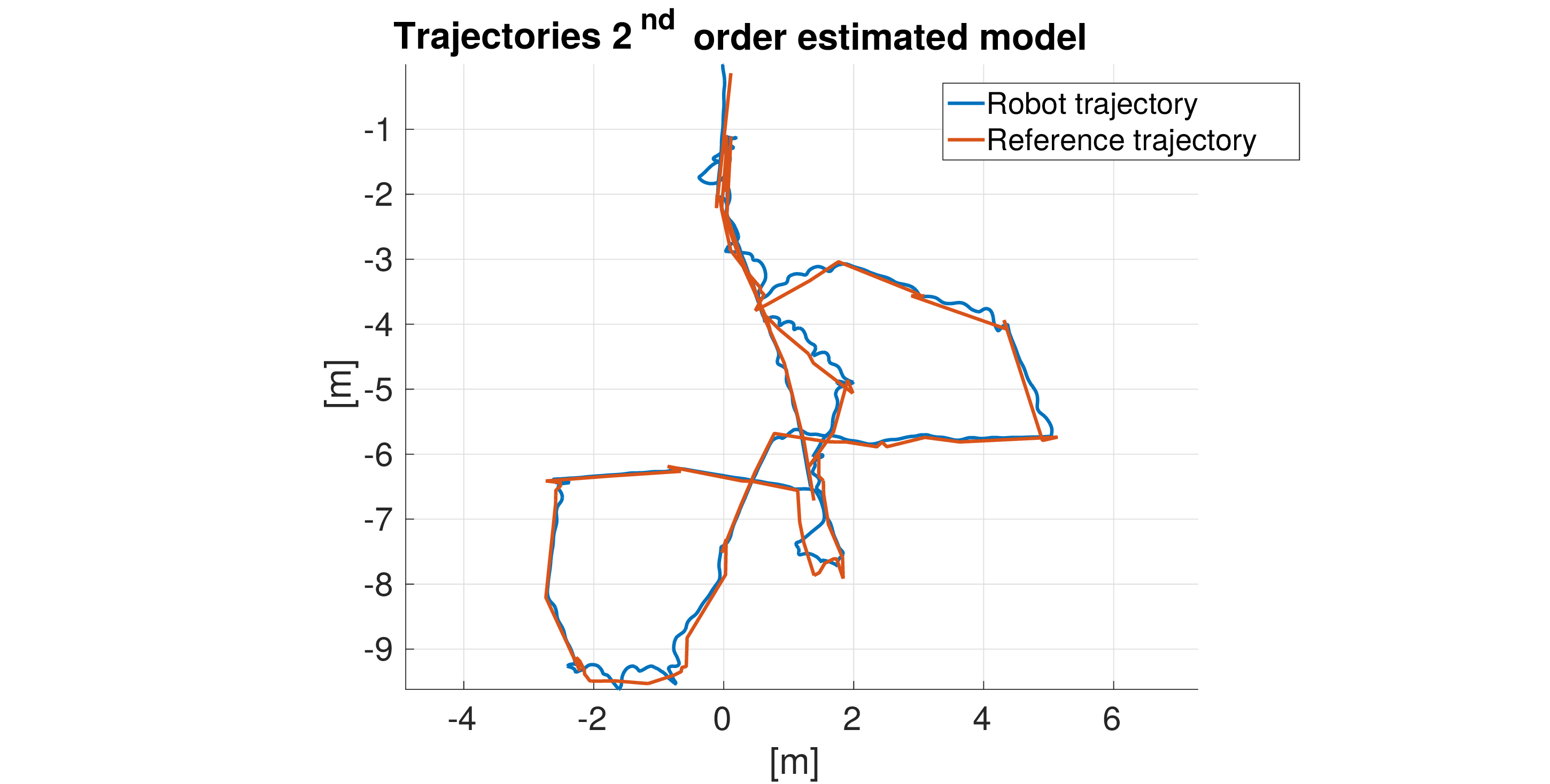}}\quad
\subfigure[]{\includegraphics[height=4.5cm]{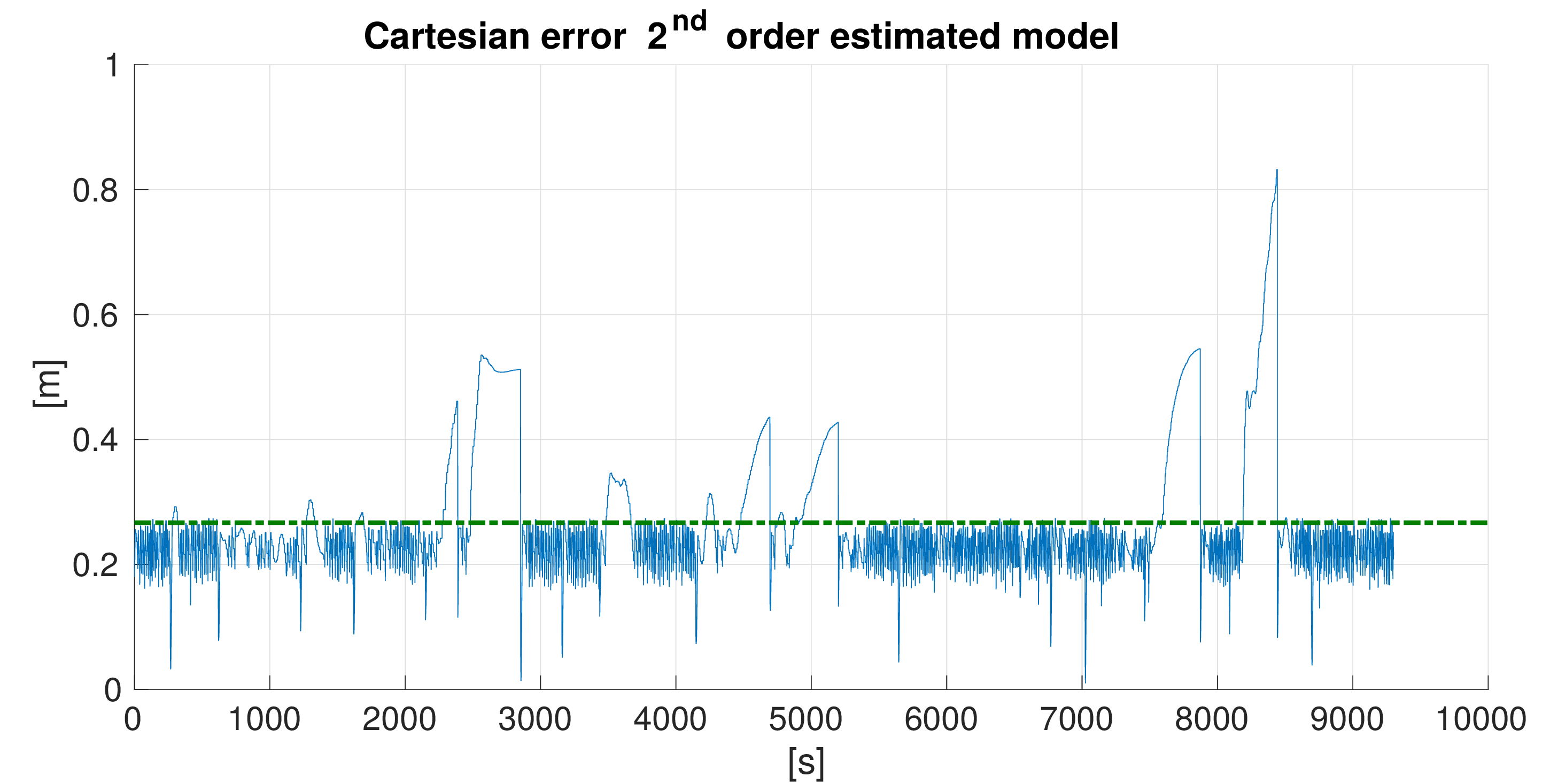}}
\caption{(a) Path effectively followed by the robot (blue line), driven by a feedback control low based on the second order model, estimated via GPs, with respect to a free-form reference trajectory (red line), on an inclined surface with friction. (b) Cartesian error of the estimated model along the reference. Error average $0.288$ $m$.}
\label{fig:case11}
\end{figure}

\begin{figure}
\centering
\subfigure[]{\includegraphics[width=\columnwidth]{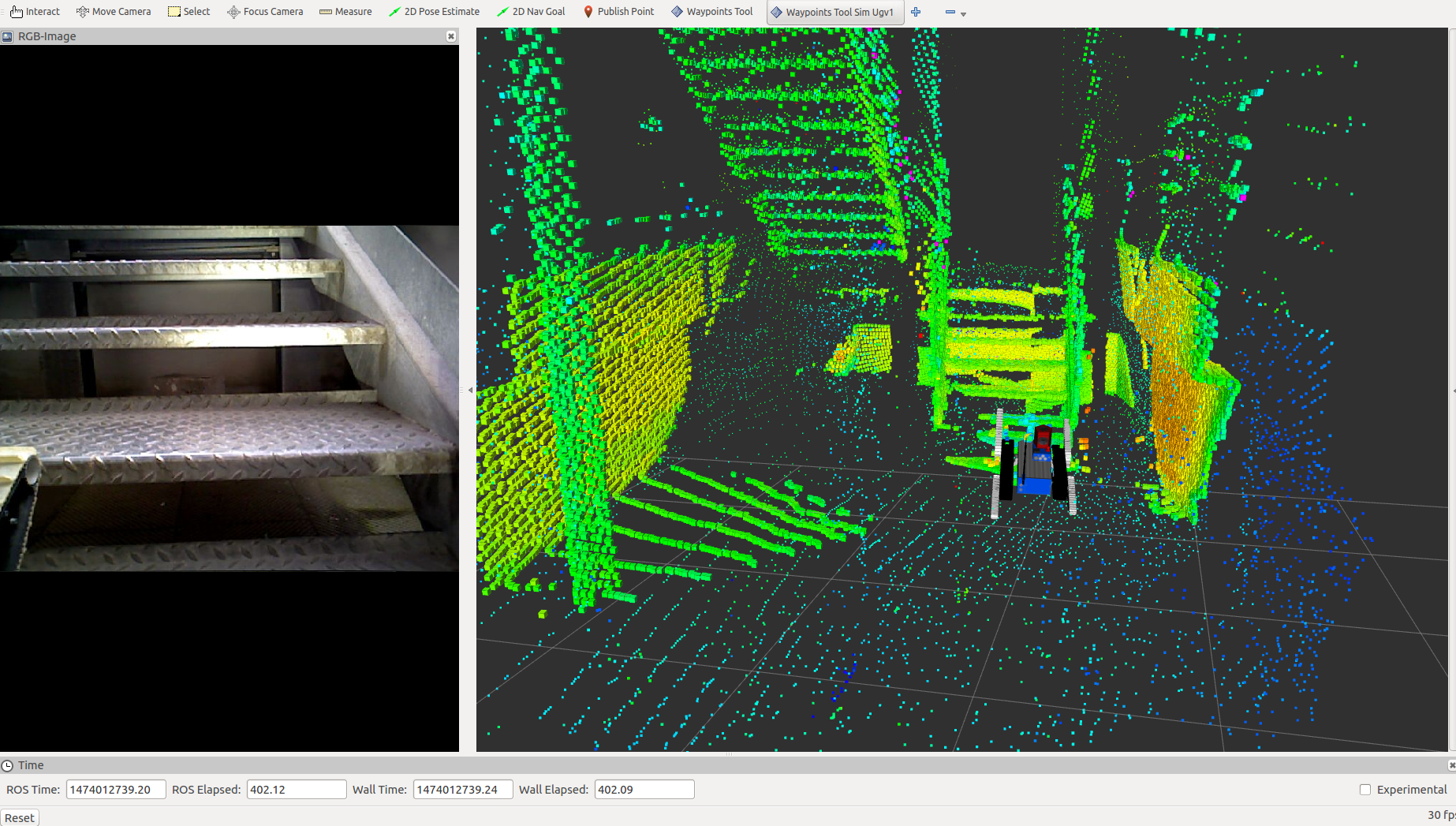}}
\subfigure[]{\includegraphics[width=\columnwidth]{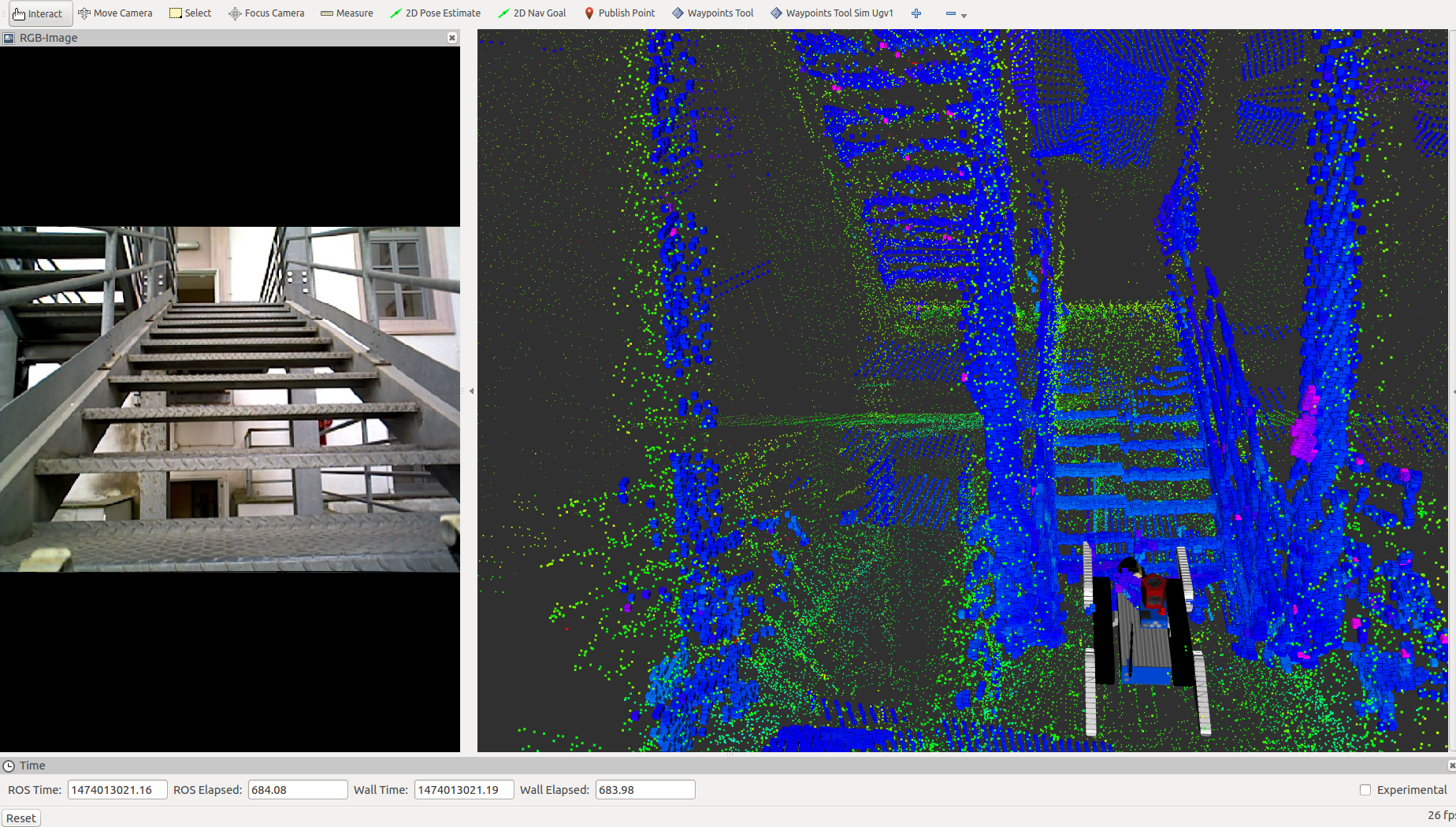}}
\caption{The tracked vehicle is climbing the metal staircase, controlled by the trajectory controller endowed with the CI-GP.}
\label{fig:stairs}
\end{figure}

\section{Results}\label{sect:experiments}
In this section we describe the experiments performed in both real and virtual environments for evaluating the prediction accuracy of the CI-GP in Section~\ref{sect:cigp} and for measuring the performance of the designed trajectory tracking controller, when the CI-GP is integrated into the control schema. For the real world experiments we provide only the prediction accuracy of the estimated model and some snapshots (see \ref{fig:stairs}), due to the lack of a ground truth to evaluate the controller performance. 

Data for estimating the parameters of the CI-GPs have been collected by giving as input to the trajectory tracking controller (see eq. (\ref{Eq:SecondOrderControlScheme})) several reference trajectories, belonging to three main classes: (1) figure-8 (see Figure~\ref{fig:case8}(a)); (2) circular (see Figure~\ref{fig:case10}(a)) and (3) free-form (see Figure~\ref{fig:case12}(a)). We also varied the inclination of the surface on top of which the reference trajectories lie (see Figure~\ref{fig:acc_pred}). In order to simulate the track-soil interaction we also varied the friction coefficient of the terrain surfaces.
During each tracking task, the measurements of the system variables have been taken according to the relation specified in eq. (\ref{Eq:SecondOrderInverseModelWithOffset}). The values of $k_{D,i}$ and $k_{P,i}$, for $i {\in} \{1,2\}$, along the diagonal of the matrices $K_D$ and $K_P$ (see eq. ({\ref{Eq:SecondOrderControlLaw}})) are set equal to $0.05$ and $0.02$, respectively. The forgetting factor $\alpha$ is equal to $0.1$ and, finally, the sample frequency $T_s{=}20$ $Hz$. 

Figure~\ref{fig:acc_pred}(a) reports the trend of the prediction accuracy of the CI-GPs on a test set extracted from data collected while the robot is tracking a figure-8 reference trajectory on a horizontal surface. Figure~\ref{fig:acc_pred}(b) shows the accuracy of the CI-GPs on a test set collected on a free-form trajectory lying on a tilted plane with friction. This accuracy is measured by computing for each point of the test set the Euclidean norm.   

After the estimation of the parameters of the CI-GP, the model has been grounded into the control schema in Figure~\ref{Fig:FigSecondOrderControlScheme}. The performance of this hybrid control law is measured with respect to the Cartesian error between the path effectively tracked by the controller under these control law on a set of test reference trajectories.

Figure~\ref{fig:case8}(b) shows the trend of the error of the controller, along a figure-8 reference trajectory (see Figure~\ref{fig:case8}(a)). 
Figure~\ref{fig:case10}(b) reports the Cartesian along a circular reference trajectory (see Figure~\ref{fig:case10}(a)).
On the other hand, Figure~\ref{fig:case12}(b) shows the performance of the hybrid control law while the robot is instructed to follow a free-form reference trajectory (see Figure~\ref{fig:case12}(a)). On the figure-8 and on the circular trajectories performance of this controller are very promising (the average error is equal to $0.1028$ $m$, $0.1821$, respectively). Conversely, on the free-form reference trajectory the performance decay. 

Figure~\ref{fig:case11} shows the performance of the controller driving the robot along a free-form reference trajectory generated on a surface, tilted of $35$ degrees, with friction coefficient $\mu_{d}{=}0.6$. In the presence of additional forces induced by the tilted reference trajectory, the controller achieves an average error of $0.288$ $m$.

Last experiment has been performed by the real robot during tracking of a reference path lying on a metal stair-case (see Figure~\ref{fig:stairs}). On this trajectory the average error of the controller is equal to $0.0233$ $m$.

\begin{figure}[t!]
\centering 
\includegraphics[width=\columnwidth]{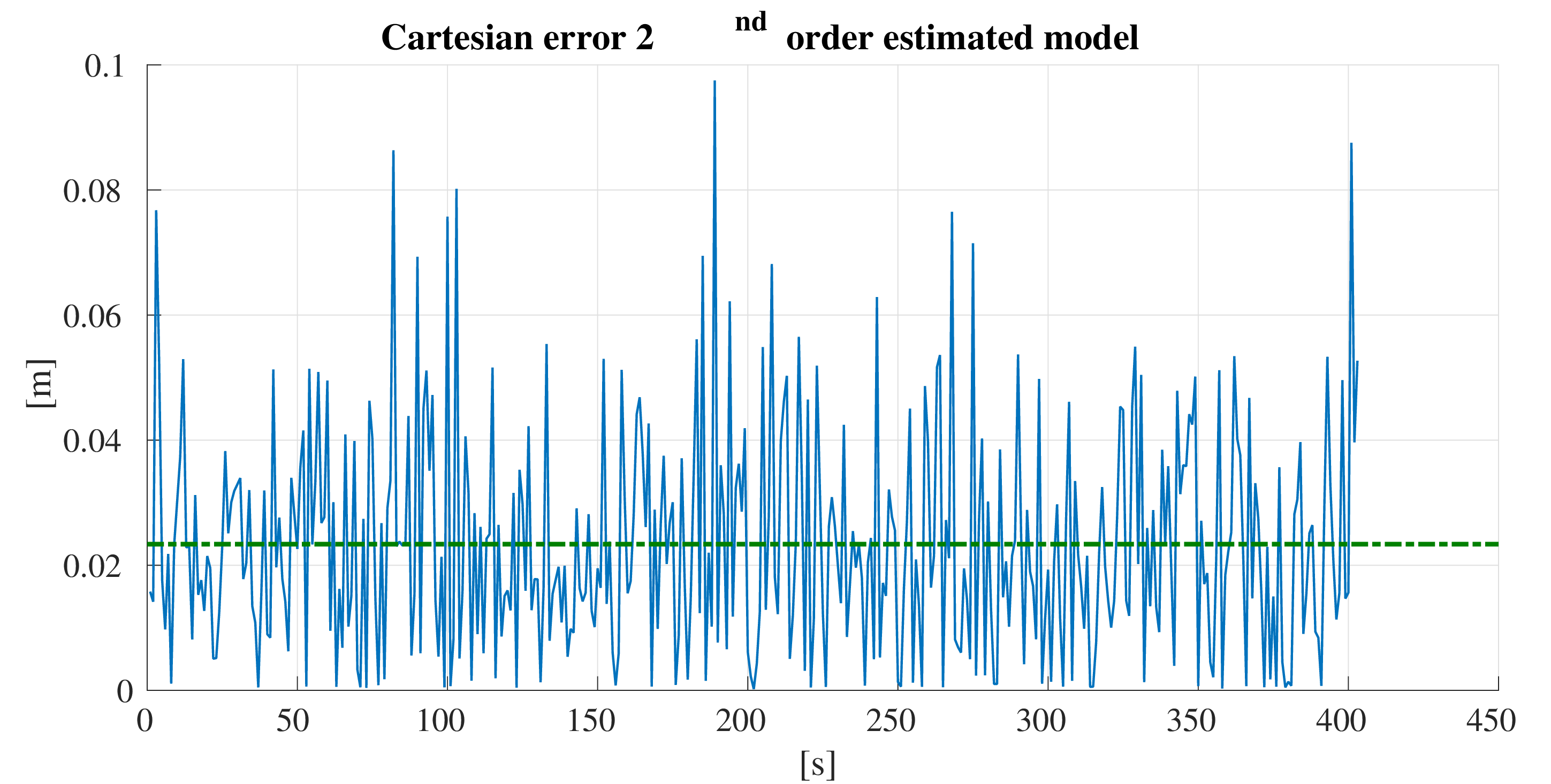}
\caption{Cartesian error of the trajectory control law grounding the CI-GP, along a trajectory passing over a metal staircase. Average error $0.0233$ $m$.}
\label{fig:acc_stairs}
\end{figure}
\typeout{CONCLUSIONS}
\section{Conclusions}

In this work we presented a new hybrid approach for designing a trajectory tracking controller for tracked vehicles. 
The advantage of the proposed approach is that it combines a great accuracy in the prediction and actual control with a real time behavior over several classes of trajectories. The system has been experimented both in real world environments, including stairs and tilted planes at various inclination degrees, and simulated environments, where ground truth is available. 

Here the nominal model has the main role of supporting the identification of that set of variables, which regulate the systems and that have to be used during the estimation process.
Conversely, the estimated dynamics integrates the feedback of the nominal model together with its prediction to cope with not modeled dynamic effects.   From the experiments, it turns out that the combination of the learned model and the control law behaves better than the nominal model in all environments, therefore this indicates that using the nominal model specifically as partial prior knowledge is a promising direction. We shall further investigate how to improve the feedback control to obtain good performance on a wider set of terrain classes.


\section{Appendix}

\section{3D Kinematic Model}\label{Sect::3DKinematicModel}

\begin{figure}[!t]
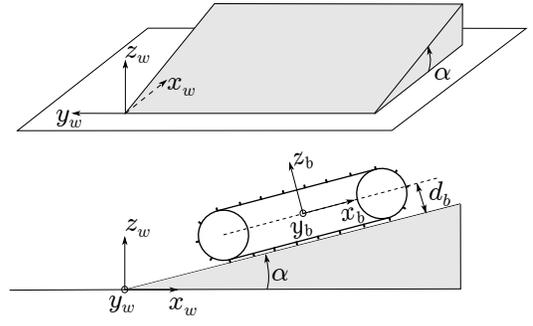

\begin{center}
\FigTiltedPlane \caption{{\em Top}: local plane with slope angle $\alpha$ and the world frame ${O_w-x_wy_wz_w}$. {\em
Bottom}: the tracked vehicle lying on the local plane $\bm{\pi}$ and the vehicle body frame $O_b-xyz$. The height $d_b$ defines the distance of the robot center from its contact surface.} \label{Fig:FigTiltedPlane}
\end{center}
\end{figure}

In this section, we consider a ground robot which moves on a 3D piecewise planar terrain. Without loss of generality\footnote{By possibly using a coordinate transformation.}, assume the robot lies on a local supporting plane $\bm{\pi}$, whose normal unit vector is $\bm{n}=[-\sin\alpha, 0 , \cos\alpha]^T$ (see Fig. \ref{Fig:FigTiltedPlane}). In this case, the plane equation is $\bm{n}^T \bm{p}=0$, where $\bm{p} \in \mathbb{R}^3$. Denote by $\bm{R}_s(\gamma) \in SO(3)$ an elementary rotation of an angle $\gamma$ about axis $s$. The homogeneous transformation matrix from robot body frame to world frame is
\begin{equation}
\bm{T}^w_b =
\begin{bmatrix}
\bm{R}^w_{b}  & \bm{p}^w_b \\
 \bm{0}^T & 1
 \end{bmatrix}
\end{equation}
where ${\bm{p}^w_b = [x, y, z]^T \in\mathbb{R}^3}$ specifies the position of the robot center w.r.t. world frame, ${\bm{R}^w_{b} = \bm{R}_z(\varphi) \bm{R}_y(\theta)\bm{R}_x(\psi) \in SO(3)}$ is the body frame rotation matrix w.r.t. world frame, and $\varphi, \theta, \psi \in  \mathbb{R}$ are respectively yaw, pitch and roll angles. One can rewrite $\bm{R}^w_{b}  = [\bm{u}_b,\bm{v}_b,\bm{w}_b]$, where $\bm{u}_b$, $\bm{v}_b$, $\bm{w}_b$ are unit vectors.
Since the robot lies on $\bm{\pi}$, one has $(i)$ $\bm{n}^T \bm{p}^w_b=d_b$  (see Fig. \ref{Fig:FigTiltedPlane}) which entails
\begin{equation}
z = \dfrac{d_b + x \sin\alpha }{\cos\alpha}
\end{equation}
and $(ii)$ $\textbf{n}^T \textbf{u}_b=0$ from which
\begin{equation}\label{Eq:ThetaEquation}
\theta = \tan^{-1}(-\tan \alpha \cos \varphi)   
\end{equation}
and $(iii)$ $\textbf{n}^T \textbf{v}_b=0$ which in turn entails
\begin{equation}
\psi = \tan^{-1}\left(   \dfrac{\tan \alpha \sin \varphi}{\tan \alpha \cos\varphi \sin \theta - \cos \theta} \right) 
\end{equation}
The previous equations show that, knowing the plane $\bm{\pi}$ representation, the vector $\bm{q} = [x,y,\varphi]^T \in SE(2)$ completely specifies the 3D robot configuration $[x,y,z,\varphi,\theta, \psi]^T \in SE(3)$. In fact, one can easily compute $z=z(\bm{q})$, $\theta=\theta(\bm{q})$ and $\psi=\psi(\bm{q})$.


\subsection{Velocities Transformations}\label{Sect:VelocitiesTransformation}

Let ${O_w-x_py_pz_p}$ be the frame obtained by rotating ${O_w-x_wy_wz_w}$ with $\bm{R}_y(-\alpha)$. This aligns the $z_w$ axis with the normal $\bm{n}$. Denote by $[x_p,y_p,z_p,\varphi_p,\theta_p, \psi_p]^T$ the robot configuration expressed w.r.t. ${O_w-x_py_pz_p}$. Since $\dot {z}_p \equiv 0$, one has
\begin{equation}
\begin{bmatrix}
\dot x\\ 
\dot y\\
\dot z 
\end{bmatrix} =
\begin{bmatrix}
\dot{x}_p  \cos\alpha \\ 
\dot{y}_p \\
\dot{x}_p \sin\alpha
\end{bmatrix}
\end{equation}
Moreover, since $\bm{R}^w_{p} = \bm{R}^w_{b} \bm{R}^b_{p} = \bm{R}_y(-\alpha) \bm{R}_z(\varphi_p)$ the following identities hold for the unit vector $\bm{u}_b$ components
\begin{equation}
\begin{bmatrix}
\cos\alpha \cos\varphi_p\\ 
\sin\varphi_p\\
\sin\alpha  \cos\varphi_p
\end{bmatrix} =
\begin{bmatrix}
\cos\varphi \cos\theta\\ 
\sin\varphi \cos\theta\\
-\sin\theta
\end{bmatrix}.
\end{equation}
These equations entail $\tan \varphi = \tan\varphi_p / \cos\alpha$ and hence
\begin{equation}
\dot \varphi = \dot\varphi_p \frac{\cos^2\varphi}{\cos\alpha\cos^2\varphi_p} = \dot\varphi_p \frac{\cos\alpha}{\cos^2\theta}
\end{equation}

\subsection{Tracked Vehicle Kinematics}\label{Sect:TrackedVehicleKinematics}

The 3D kinematic model of a vehicle moving on the plane $\bm{\pi}$ can be described by the following equations \cite{Endo:2007}
\begin{align}
\dot x_p &= \frac{v_r(1-a_r) + v_l(1-a_l)}{2} \cos\varphi_p \\
\dot y_p &= \frac{v_r(1-a_r) + v_l(1-a_l)}{2} \sin\varphi_p \\
\dot \varphi_p &=  \frac{v_r(1-a_r) - v_l(1-a_l)}{d}.
\end{align}
By using the equations reported in Sect.~\ref{Sect:VelocitiesTransformation}, one has in the world frame 
\begin{align}
\dot x &= \frac{v_r(1-a_r) + v_l(1-a_l)}{2}\cos\varphi \cos \theta \\
\dot y &= \frac{v_r(1-a_r) + v_l(1-a_l)}{2} \sin\varphi \cos\theta \\
\dot \varphi &=  \frac{v_r(1-a_r) - v_l(1-a_l)}{d} \frac{\cos\alpha}{\cos^2\theta}
\end{align}
where $\theta$ is computed by using eq.~(\ref{Eq:ThetaEquation}).
In particular, 
the \emph{slip ratios}  $a_r, a_l \in  \mathbb{R}$ are defined as
\begin{eqnarray}\label{Eq:SlipRatios}
a_r = \frac{v_r - v^{\prime}_r}{v_r}\\
a_l = \frac{v_l - v^{\prime}_l}{v_l}
\end{eqnarray}
The quantities $v^{\prime}_r, v^{\prime}_l \in  \mathbb{R}$ are respectively the actual longitudinal velocities of the right and left tracks. Track $i$ slips if $v^{\prime}_i > v_i$, it skids otherwise. If no slippage occurs, one has $v_r = v^{\prime}_r$, $v_l = v^{\prime}_l$ and thus $a_r = 0$, $a_l = 0$. Therefore, the slip ratios quantitatively characterize the longitudinal slips of the left and right tracks. These are estimated as proposed in \cite{Endo:2007} starting from the actual angular velocity of the vehicle (as estimated from gyro-sensor data) and assuming a given relationship model $a_r/a_l = F(v_r,v_l)$ which is typically specialized as
\begin{equation}
a_r/a_l = -{\rm sign}(v_l \cdot v_r)\left| \dfrac{v_l}{v_r} \right| ^ n
\end{equation}
where $n$ is a parameter that depends on many factors and is experimentally estimated. 

In practice, the lateral slip of the vehicle can be described by using the \emph{slip angle} 
\begin{equation}\label{Eq:SlipAngle}
\beta = \tan^{-1}(v_{by}/v_{bx}).
\end{equation}
where $v_{bx}, v_{by} \in  \mathbb{R}$ are respectively the longitudinal and lateral components of the absolute vehicle velocity w.r.t. its body frame. It is worth noting that the slip angle can also be included in the vehicle kinematic model and can be online estimated by using an EKF \cite{Le:1997}.

Equations (\ref{Eq:SlipRatios})--(\ref{Eq:SlipAngle}) can be used in order to model the actual vehicle slip motions.


\bibliographystyle{IEEEtran}
\bibliography{bib/bibliography}

\begin{thebibliography}{10}
\providecommand{\url}[1]{#1}
\csname url@rmstyle\endcsname
\providecommand{\newblock}{\relax}
\providecommand{\bibinfo}[2]{#2}
\providecommand\BIBentrySTDinterwordspacing{\spaceskip=0pt\relax}
\providecommand\BIBentryALTinterwordstretchfactor{4}
\providecommand\BIBentryALTinterwordspacing{\spaceskip=\fontdimen2\font plus
\BIBentryALTinterwordstretchfactor\fontdimen3\font minus
  \fontdimen4\font\relax}
\providecommand\BIBforeignlanguage[2]{{%
\expandafter\ifx\csname l@#1\endcsname\relax
\typeout{** WARNING: IEEEtran.bst: No hyphenation pattern has been}%
\typeout{** loaded for the language `#1'. Using the pattern for}%
\typeout{** the default language instead.}%
\else
\language=\csname l@#1\endcsname
\fi
#2}}

\bibitem{Martinez:2005}
J.~L. Mart\'{\i}nez, A.~Mandow, J.~Morales, S.~Pedraza, and
  A.~Garc\'{\i}a-Cerezo, ``Approximating kinematics for tracked mobile
  robots,'' \emph{Int. J. Rob. Res.}, vol.~24, no.~10, pp. 867--878, Oct. 2005.

\bibitem{Ishigami:2006}
G.~Ishigami, K.~Nagatani, and K.~Yoshida, ``Path following control with slip
  compensation on loose soil for exploration rover,'' in \emph{IROS}, 2006, pp.
  5552--5557.

\bibitem{Endo:2007}
D.~Endo, Y.~Okada, K.~Nagatani, and K.~Yoshida, ``Path following control for
  tracked vehicles based on slip-compensating odometry,'' in \emph{IROS}, 2007,
  pp. 2871--2876.

\bibitem{Moosavian:2008}
S.~A.~A. Moosavian and A.~Kalantari, ``Experimental slip estimation for exact
  kinematics modeling and control of a tracked mobile robot,'' in \emph{IROS},
  2008, pp. 95--100.

\bibitem{Dar:2010}
T.~M. Dar and R.~G. Longoria, ``Slip estimation for small-scale robotic tracked
  vehicles,'' in \emph{ACC}, 2010, pp. 6816--6821.

\bibitem{Williams:2008}
C.~Williams, S.~Klanke, V.~Sethu, and K.~M. Chai, ``Multi-task gaussian process
  learning of robot inverse dynamics,'' in \emph{Adv. Neural. Inf. Process
  Syst.}, 2009, pp. 265--272.

\bibitem{Deisenroth:2015}
M.~P. Deisenroth, D.~Fox, and C.~E. Rasmussen, ``Gaussian processes for
  data-efficient learning in robotics and control,'' \emph{IEEE Transactions on
  Pattern Analysis and Machine Intelligence}, vol.~37, no.~2, pp. 408--423,
  2015.

\bibitem{Ko:2007}
J.~Ko, D.~J. Klein, D.~Fox, and D.~Haehnel, ``Gaussian processes and
  reinforcement learning for identification and control of an autonomous
  blimp,'' in \emph{ICRA}, 2007, pp. 742--747.

\bibitem{Ostafew:2014}
C.~J. Ostafew, A.~P. Schoellig, and T.~D. Barfoot, ``Learning-based nonlinear
  model predictive control to improve vision-based mobile robot path-tracking
  in challenging outdoor environments,'' in \emph{ICRA}, 2014, pp. 4029--4036.

\bibitem{Ng:2000}
A.~Y. Ng and M.~Jordan, ``Pegasus: A policy search method for large mdps and
  pomdps,'' in \emph{Proceedings of the Sixteenth Conference on Uncertainty in
  Artificial Intelligence}, ser. UAI'00.\hskip 1em plus 0.5em minus 0.4em\relax
  San Francisco, CA, USA: Morgan Kaufmann Publishers Inc., 2000, pp. 406--415.

\bibitem{Siciliano:2010}
B.~Siciliano, L.~Sciavicco, L.~Villani, and G.~Oriolo, \emph{Robotics:
  modelling, planning and control}.\hskip 1em plus 0.5em minus 0.4em\relax
  Springer Science \& Business Media, 2010.

\bibitem{Oriolo:2002}
G.~Oriolo, A.~De~Luca, and M.~Vendittelli, ``Wmr control via dynamic feedback
  linearization: design, implementation, and experimental validation,''
  \emph{IEEE Transactions on control systems technology}, vol.~10, no.~6, pp.
  835--852, 2002.

\bibitem{Holoborodko:2008}
P.~Holoborodko, ``Smooth noise robust differentiators,''
  http://www.holoborodko.com/pavel/numerical-methods/numerical-derivative/smooth-low-noise-differentiators/,
  2008.

\bibitem{Isidori:1995}
A.~Isidori, \emph{Nonlinear Control Systems}, 3rd~ed.\hskip 1em plus 0.5em
  minus 0.4em\relax Springer-Verlag New York, Inc., 1995.

\bibitem{Rasmussen:2005}
C.~E. Rasmussen and C.~K.~I. Williams, \emph{Gaussian processes for machine
  Learning}.\hskip 1em plus 0.5em minus 0.4em\relax The MIT Press, 2005.

\bibitem{Le:1997}
A.~Tuan~Le, D.~C. Rye, and H.~F. Durrant-Whyte, ``Estimation of track-soil
  interactions for autonomous tracked vehicles,'' in \emph{ICRA}, vol.~2, 1997,
  pp. 1388--1393.

\end{thebibliography}

\begin{acronym}
\acro{CIGP}{Conditional Independent Gaussian Process}
\end{acronym}

\end{document}